\documentclass{article}

\usepackage{iclr2026_conference,times}
\usepackage{subcaption}
\usepackage{algorithm}
\usepackage{algorithmic}
\usepackage{graphicx}
\usepackage{siunitx}

\usepackage[utf8]{inputenc} 
\usepackage[T1]{fontenc}    
\usepackage{hyperref}       
\usepackage{url}            
\usepackage{booktabs}       
\usepackage{amsfonts}       
\usepackage{nicefrac}       
\usepackage{microtype}      
\usepackage{xcolor}         
\usepackage{amsmath}
\usepackage{multirow}
\usepackage{amsthm}
\usepackage{cleveref}
\usepackage{wrapfig}
\usepackage{enumitem}
\theoremstyle{plain}
\newtheorem{theorem}{Theorem}[section]
\newtheorem{proposition}[theorem]{Proposition}

\theoremstyle{definition}

\theoremstyle{remark}

\newcolumntype{L}[1]{>{\raggedright\arraybackslash}p{#1}}  
\newcolumntype{C}[1]{>{\centering\arraybackslash}p{#1}} 
\iclrfinalcopy

\title{Machine Unlearning under Retain–Forget Entanglement}

\author{%
  Jingpu Cheng\\
      Department of Mathematics \\
   National University of Singapore\\
  \texttt{chengjingpu@u.nus.edu} \\
 \And
  \hspace{1.9cm} Ping Liu \\
\hspace{1.9cm} Department of Computer Science\\
\hspace{1.9cm} University of Nevada Reno\\
\hspace{1.9cm} \texttt{pino.pingliu@gmail.com} \\
  \And
 \\
    \And
  Qianxiao Li \\
  Department of Mathematics \\
 Institute for Functional Intelligent Materials \\
 National University of Singapore \\
  \texttt{qianxiao@nus.edu.sg} \\
  \And
    Chi Zhang\thanks{Corresponding: \texttt{czhang24@nus.edu.sg}} \\
  Department of Mathematics \\
 National University of Singapore \\
  \texttt{czhang24@nus.edu.sg} \\
}
\begin{document}

\maketitle

\begin{abstract}
Forgetting a subset in machine unlearning is rarely an isolated task. Often, retained samples that are closely related to the forget set can be unintentionally affected, particularly when they share correlated features from pretraining or exhibit strong semantic similarities. To address this challenge, we propose a novel two-phase optimization framework specifically designed to handle such retain–forget entanglements. In the first phase, an augmented Lagrangian method increases the loss on the forget set while preserving accuracy on less-related retained samples. The second phase applies a gradient projection step, regularized by the Wasserstein-2 distance, to mitigate performance degradation on semantically related retained samples without compromising the unlearning objective. We validate our approach through comprehensive experiments on multiple unlearning tasks, standard benchmark datasets, and diverse neural architectures, demonstrating that it achieves effective and reliable unlearning while outperforming existing baselines in both accuracy retention and removal fidelity. Our code is available \href{https://github.com/Jingpu-Cheng/unlearning-entanglement}{here}.

\end{abstract}

\section{Introduction}
The indelible memory of machine learning systems presents a paradoxical challenge: what happens when we need algorithms to forget? Consider a face recognition system deployed for secure access. When an employee resigns, their biometric signature cannot simply be deactivated—it must be completely expunged from the underlying machine learning model. This ability to erase specific information extends to the broader concept known as ``machine unlearning''~\citep{cao2015towards}, which aims to selectively remove the impact of specific data from trained models. In fact, the importance of unlearning extends beyond the general concept itself, with critical applications in meeting legal obligations~\citep{mantelero2013eu}, mitigating harmful representational biases~\citep{mehrabi2021survey}, and repairing models from mislabeled or poisoned training data~\citep{northcutt2021pervasive}.

Recent work has investigated a range of unlearning scenarios, including random-sample unlearning~\citep{golatkar2020eternal, izzo2021approximate}, class-wise unlearning~\citep{kurmanji2023towards}, and concept-level unlearning, where the forget set does not necessarily align with class labels~\citep{zhu2024decoupling}. More recently, several methods~\citep{foster2024fast, seo2025revisiting, xu2024don} have introduced efficient post-hoc or feature-space–aware solutions. Together, these approaches have significantly advanced our understanding of what it means for a model to “forget”.

Yet, forgetting is rarely an isolated task. Removing the influence of one group of data often directly affects another group that is closely correlated with it. For instance, forgetting toxic statements involving a minority group may inadvertently alter the model’s behavior on non-toxic statements about the same group~\citep{shen2024camu}. Similarly, forgetting one subclass of images within a broader category can disrupt predictions on closely related subclasses~\citep{fan2024challenging}. Existing works typically assess retain performance by averaging over the entire retain set, paying little attention to these sensitive, correlated subsets, where performance is both more fragile and more consequential.
We therefore focus on the challenge of \emph{retain–forget entanglement}, where certain retained samples are closely tied to the forget set and particularly susceptible to unintended degradation. To mitigate the resulting performance drops in these sensitive subsets, we propose a two-stage framework based on constrained optimization. In the first stage, an augmented Lagrangian method enforces forgetting by increasing the loss on the forget set while preserving accuracy on less-correlated retained samples. In the second stage, the model is refined through gradient projection to restore performance on retained samples that are more strongly correlated with the forget set, without compromising the forgetting objective. To further stabilize the process and enhance generalization, we also regularize the loss distribution using the Wasserstein-2 distance during this stage.

We evaluate our method across a variety of subclass-level unlearning scenarios, covering diverse forgetting tasks, multiple neural network architectures, and standard benchmark datasets. The results demonstrate that our approach consistently achieves effective forgetting while maintaining high accuracy on the retained data. In structured selective unlearning settings, it significantly outperforms prior methods, demonstrating robustness and reliability without compromising the intended forgetting effect. Importantly, it preserves performance on retained samples that are closely related to the forget set, ensuring that sensitive subsets remain largely unaffected.

Our contributions can be summarized as follows:
\begin{itemize}[itemsep=0pt, topsep=0pt, parsep=0.5pt, partopsep=0pt]
    \item \textbf{Highlighting retain--forget entanglement:} We focus on a correlation-aware unlearning setting, where the forget set is entangled with another group of data. This setting better reflects real-world unlearning demands and introduces new technical challenges due to significant distributional overlap with the retained data.
    \item \textbf{A novel two-stage unlearning framework:} We propose a two-stage optimization-based framework to address this challenge. The first stage uses an augmented Lagrangian method to enforce forgetting while preserving performance on less-correlated samples. The second stage applies gradient projection with Wasserstein-2 distance regularization to recover performance on sensitive retained samples without compromising the forgetting objective.
    \item \textbf{Comprehensive evaluation:} We provide a comprehensive empirical evaluation across diverse tasks, architectures, and datasets, demonstrating that our method achieves strong forgetting performance while retaining accuracy on preserved data.
\end{itemize}

\section{Related Works}
\vspace{-0.1in}
\paragraph{Constrained Optimization in Machine Learning}
Constrained optimization is widely used in machine learning to enforce domain-specific requirements like fairness and safety~\citep{cotter2019optimization,zafar2019fairness,achiam2017constrained,liu2022constrained}. In fairness-aware learning, these constraints prevent discriminatory predictions and are naturally framed as optimization problems~\citep{donini2018empirical,zafar2019fairness,caton2024fairness}. Classical techniques, such as penalty methods~\cite{berk2017convex} and Lagrangian-based approaches~\citep{cruzfairgbm,celis2019classification,cotter2019optimization,lokhande2020fairalm}, have proven effective in these settings. Similarly, in reinforcement learning, safety constraints guide agents away from risky actions~\citep{chow2018risk,liu2022constrained}, often handled through primal-dual optimization to penalize constraint violations~\citep{achiam2017constrained,liang2018accelerated,bohez2019value}.

\vspace{-0.15in}

\paragraph{Machine Unlearning}
The concept of machine unlearning was formalized by~\citep{cao2015towards}, requiring model outputs indistinguishable from retraining without the deleted data. However, full retraining is often infeasible for large-scale models, motivating approximate methods~\citep{golatkar2020eternal,golatkar2020forgetting,izzo2021approximate,thudi2022unrolling,mehta2022deep}. Many build on the framework of~\cite{ginart2019making}, gradient-based updates~\citep{golatkar2020eternal,fan2024salun,patel2025learning}, sparsity-based pruning~\citep{jia2023model}, prompt editing~\citep{liu2024large}, fisher and influence based methods~\citep{foster2024fast,shi2024deepclean,wu2022puma} and adversarial approaches~\citep{di2024adversarial}. In particular, fine-tuning~\citep{warnecke2021machine,zhang2024parameter,zhangweight} has been shown to be an effective post-training approach and is widely used in machine unlearning. For example, \cite{kurmanji2023towards} proposed post-training methods by discouraging the model from correctly predicting labels on a forget set. Meanwhile, retain–forget entanglement has been show to impact unlearning, where accuracy drops are often concentrated on retained examples most similar to the forget set~\citep{zhao2024whatmakes,chang2025whichretain}. For example, 
subclass-level forgetting~\citep{zhu2024decoupling,foster2024fast,seo2025revisiting} considers settings where the forget subclass is semantically close to other subclasses. Yet performance is often reported as an average over the entire retain set, which can mask degradation on the correlated subset. Similarly, in LLM unlearning~\citep{maini2024tofu, jin2024rwku,chang2025whichretain,choi2025optout}, a neighbor set is often used for evaluation or regularization; in concept erasure for generative models~\citep{gandikota2024unified,xie2025erasing,liu2025erased}, preservation sets are used to maintain overall performance. Yet, these sets need not be closely related to the erased targets, whereas we explicitly decompose the retain set into adjacent and remote subsets and report the performance on both of them.

\section{Problem Formulation}

Let $\mathcal{D} = \{(x_i, y_i)\}_{i=1}^N$ be a dataset of $N$ samples, where $x_i \subset \mathcal{X}$ denotes an input and $y_i \in \mathcal{Y}$ is its corresponding label. Let $f_{\theta_0}(x)$ be a model trained on $\mathcal{D}$ with parameters $\theta_0$. Given a subset $\mathcal{D}_f \subset \mathcal{D}$, the goal of machine unlearning is to obtain updated parameters $\tilde{\theta}$ such that the resulting model $f_{\tilde{\theta}}(x)$ effectively forgets $\mathcal{D}_f$, while preserving performance on the remaining data $\mathcal{D}_r := \mathcal{D} \setminus \mathcal{D}_f$.

Classical formulations of machine unlearning typically do not assume further structures in the retain dataset. However, in many applications, forgetting $\mathcal{D}_f$ affects not only average performance on $\mathcal{D}_r$, but disproportionately impacts a correlated portion inside $\mathcal{D}_r$~\citep{fan2024challenging}. We therefore
conceptually split the retain set into two parts:
\[
\mathcal{D}_r=\mathcal{D}_r^{\text{adj}} \ \cup\  \mathcal{D}_r^{\text{rem}}, 
\qquad \mathcal{D}_r^{\text{adj}}\cap \mathcal{D}_r^{\text{rem}}=\emptyset.
\]
Here, the adjacent retain set $\mathcal{D}_r^{\text{adj}}$ consists of retained examples that are correlated with $\mathcal{D}_f$ and thus more sensitive to forgetting, while the remote retain set $\mathcal{D}_r^{\text{rem}}$ comprises the remaining, less-related retained examples, which we refer to as the remote samples.

In practice, the entanglement between the forget set and retained samples can arise from different sources. One common scenario is subclass-level unlearning, where the forget set constitutes a fine-grained subclass within a broader class. For example, if a model is trained on the 20 superclasses of CIFAR-100 and the forget set consists of one subclass, we can define $\mathcal{D}_r^{\text{adj}}$ as the remaining samples from the same superclass and $\mathcal{D}_r^{\text{rem}}$ as the rest of the dataset. Another scenario occurs when retained samples form a semantically related group with the forget set. For instance, in a language dataset containing normal and offensive sentences, comments referring to the same group of people may be strongly correlated with the forget set.


The goals of this retain-forget entangled machine unlearning are therefore to obtain an updated model such that

\begin{enumerate}[itemsep=0pt, topsep=0pt, parsep=0pt, partopsep=0pt]
    \item The model retains its performance on $\mathcal{D}_r$, especially on samples belonging to $D_r^{\text{adj}}$ that have strong correlation to $\mathcal{D}_f$.
    
    \item The model forgets the forget set $\mathcal{D}_f$ by removing or mitigating its influence.
\end{enumerate}

It is important to note that the definition of ``forgetting'' can vary depending on the application. In privacy-focused contexts~\citep{cao2015towards}, the objective is often for $f_{\tilde{\theta}}$ to emulate a model retrained from scratch on the retained set $\mathcal{D}_r$. In contrast, there are scenarios that prioritize maximally reducing the model's performance on the forget set $\mathcal{D}_f$, as studied in~\citep{choi2023towards}. This approach is particularly relevant when $\mathcal{D}_f$ contains undesirable patterns, such as social biases, offensive content, or behaviors subject to withdrawal requests, where the goal is for the model to completely disregard the influence of these samples. In this work, we adopt the latter perspective.

\section{Methods}
Machine unlearning naturally poses a multi-objective challenge: removing the influence of the forget set while maintaining overall performance. In the retain-forget entangled setting, this becomes more difficult due to the semantic and distributional entanglement between the forget set $\mathcal{D}_f$ and the strongly correlated retain set $\mathcal{D}_r^{\text{adj}}$. To address this challenge, we introduce a two-stage optimization framework in this section.


\subsection{Stage 1: Forgetting via Controlled Optimization}

The first stage of our framework aims to aggressively increase the loss on the forget set while preventing substantial degradation on the less-related retain set.
Formally, let \textcolor{black}{$\mathcal{L}_f(\theta): = \mathcal{L}_f(\theta; \mathcal{D}_f)$, $\mathcal{L}_r^{\text{adj}}(\theta):= \mathcal{L}_r^{\text{adj}}(\theta; \mathcal{D}_r^{\text{adj}})$, and $\mathcal{L}_r^{\text{rem}}(\theta):= \mathcal{L}_r^{\text{rem}}(\theta; \mathcal{D}_r^{\text{rem}})$} denote the losses on $\mathcal{D}_f$, $\mathcal{D}_r^{\text{adj}}$, and $\mathcal{D}_r^{\text{rem}}$, and let $\theta_0$ be the parameters of the original model.
We formulate Stage~1 as the constrained optimization problem
\begin{equation}
    \min_{\theta} \; -\mathcal{L}_f(\theta)  \quad 
    \text{subject to} \quad 
    \mathcal{L}_r^{\text{rem}}(\theta) = \mathcal{L}_r^{\text{rem}}(\theta_0).
\end{equation}

We adopt an augmented Lagrangian formulation~\citep{bertsekas2014constrained} to provide an adaptive way of balancing the objective and the constraint:
\begin{equation}
    \label{eq:stage1}
    \mathcal{L}_{\text{aug}}(\theta; \lambda, \mu) = 
    -\mathcal{L}_f(\theta) 
    + \lambda \big(\mathcal{L}_{r}^{\text{rem}}(\theta) - \mathcal{L}_{r}^{\text{rem}}(\theta_0)\big) 
    + \frac{\mu}{2}\big(\mathcal{L}_{r}^{\text{rem}}(\theta) - \mathcal{L}_{r}^{\text{rem}}(\theta_0)\big)^2,
\end{equation}
where $\lambda$ is the Lagrange multiplier and $\mu>0$ is a penalty coefficient.
We initialize $\lambda=0$ and iteratively update $\theta$ via gradient descent,
\begin{equation}
    \theta \leftarrow \theta - \eta \nabla_\theta \mathcal{L}_{\text{aug}}(\theta; \lambda, \mu),
\end{equation}
followed by updating the multiplier according to constraint violation:
\begin{equation}
    \label{eq:adaptive_lambda}
    \lambda \leftarrow \lambda + \mu \big(\mathcal{L}_{r}^{\text{rem}}(\theta) - \mathcal{L}_{r}^{\text{rem}}(\theta_0)\big).
\end{equation}

This iterative update scheme adaptively tightens or relaxes the penalty as needed, avoiding the need to manually tune a fixed trade-off coefficient. The objective of Stage~1 is to enforce unlearning on the forget set $\mathcal{D}_f$ while preserving performance on the less-related retained subset $\mathcal{D}_r^{\text{rem}}$.

\subsection{Stage 2: $W_2$-Distance Guided Projected Gradient Descent (W-PGD)}
Importantly, we refrain from explicitly optimizing over the strongly correlated retain set $\mathcal{D}_r^{\text{adj}}$ in Eq~\eqref{eq:stage1} to avoid conflicting gradients (see Appendix~\ref{appendix:ablation}). As a result, the model achieves low accuracy on the forget set $\mathcal{D}_f$ while maintaining strong performance on the remote retain set $\mathcal{D}_r^{\text{rem}}$. However, due to the semantic or distributional overlap between $\mathcal{D}_f$ and $\mathcal{D}_r^{\text{adj}}$, performance on the adjacent retain set $\mathcal{D}_r^{\text{adj}}$ typically degrades. The objective of the second stage is to restore the model’s accuracy on $\mathcal{D}_r^{\text{adj}}$ while preserving the performance on $\mathcal{D}_f$ and $\mathcal{D}_r^{\text{rem}}$.

\begin{figure}[b]
    \centering
    \begin{subfigure}{0.252\textwidth}
        \centering
        \includegraphics[width=\linewidth]{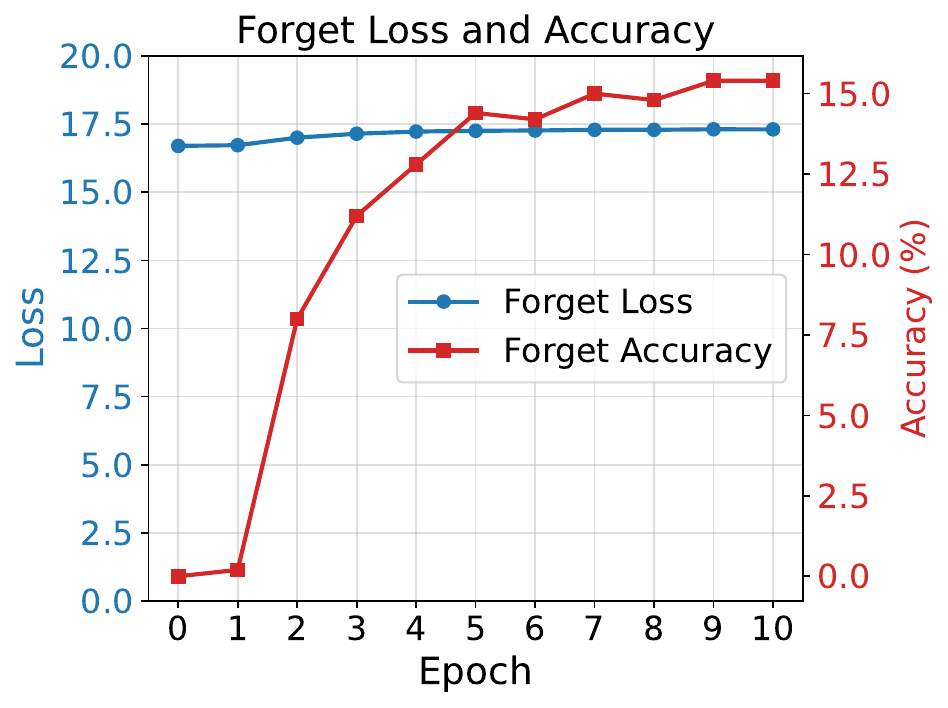}
        \vspace{-0.15in}
        \subcaption{}
        \label{fig:sub1}
    \end{subfigure}
    \begin{subfigure}{0.24\textwidth}
        \centering
        \includegraphics[width=\linewidth]{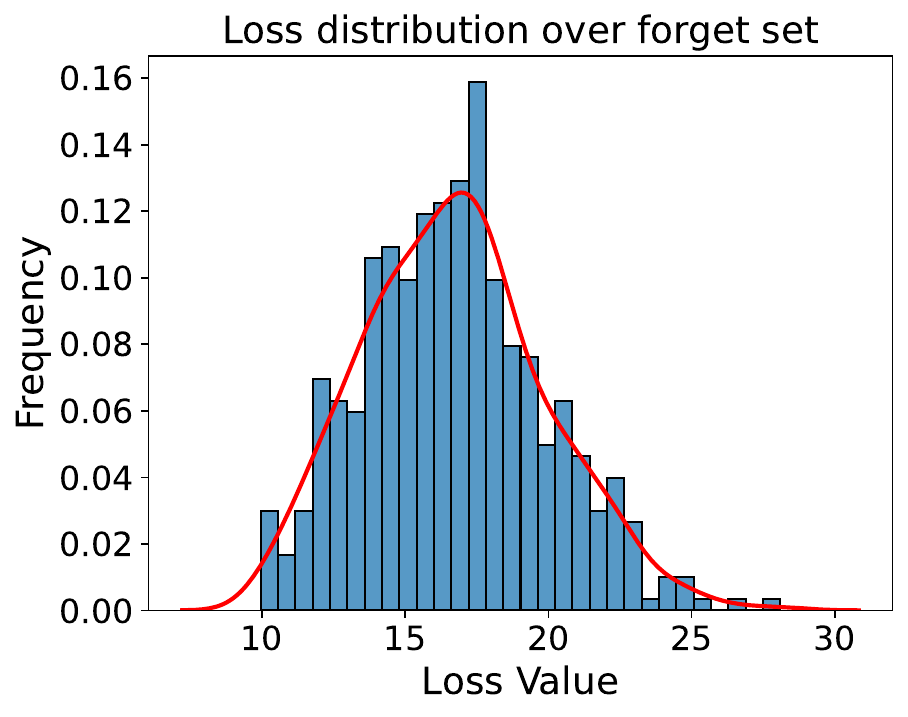}
        \vspace{-0.15in}
        \subcaption{}
        \label{fig:sub2}
    \end{subfigure}
    \begin{subfigure}{0.24\textwidth}
        \centering
        \includegraphics[width=\linewidth]{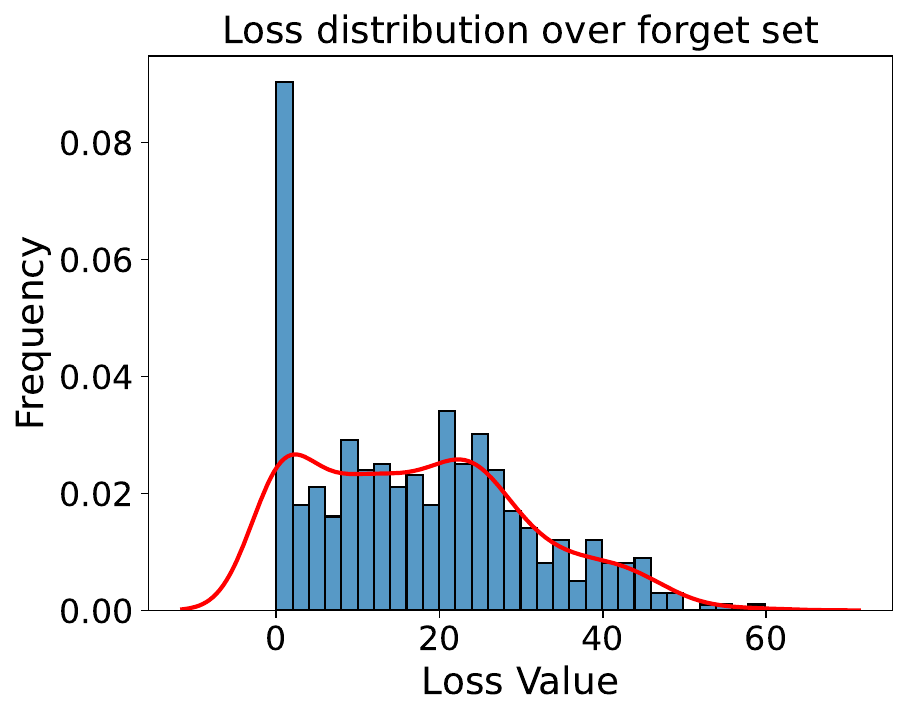}
        \vspace{-0.15in}
        \subcaption{}
        \label{fig:sub3}
    \end{subfigure}
    \begin{subfigure}{0.24\textwidth}
        \centering
        \includegraphics[width=\linewidth]{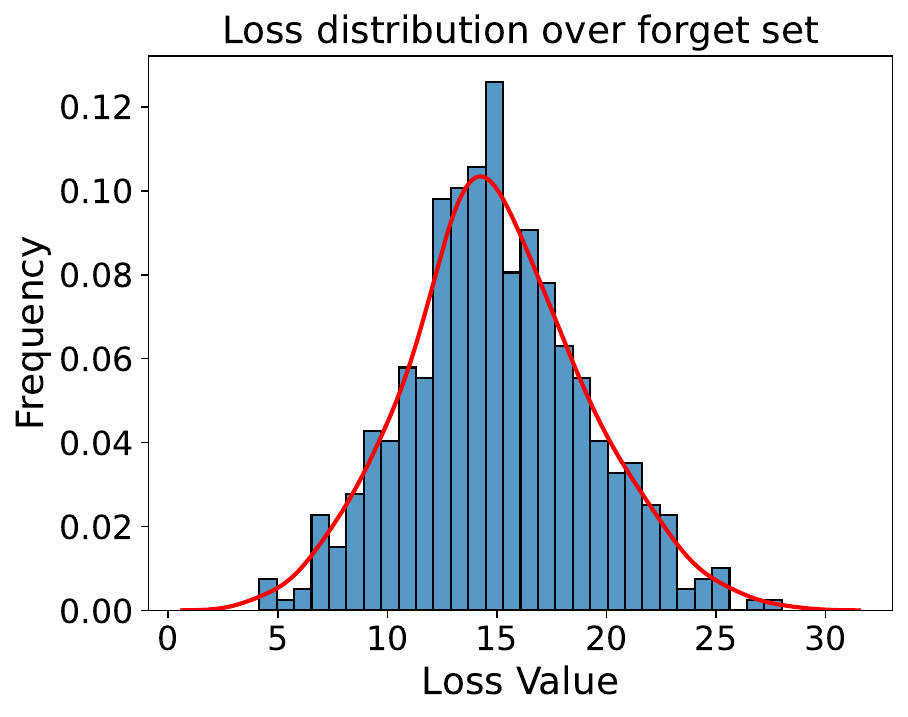}
        \vspace{-0.15in}
        \subcaption{}
        \label{fig:sub4}
    \end{subfigure}
    \caption{Training dynamics of PGD and cross-entropy loss distributions on $\mathcal{D}_f$. (a) Loss and accuracy curves of PGD during the second stage; (b) Original loss distribution on $\mathcal{D}_f$ after the first stage; (c) Loss distribution on $\mathcal{D}_f$ after applying PGD in the second stage; (d) Loss distribution on $\mathcal{D}_f$ after applying W-PGD. Comparing figure (b) and (c), PGD notably \textbf{skews the loss distribution}, with some samples attaining near-zero loss. In contrast, W-PGD (d) preserves a distribution closer to the original and effectively avoids assigning low loss to forget set samples.
}
    \label{fig:all_images}
\end{figure}

\subsubsection{Is Classical Projected Gradient Descent Good Enough?}
\label{sec:PGD}
We begin by aiming to improve the performance on the adjacent retain set using the classical Projected Gradient Descent (PGD) framework~\citep{bertsekas1999nonlinear}, but adopt its first-order (linearized) projection variant, as widely used in multi-task learning~\citep{yu2020gradient,farajtabar2020orthogonal}. 
In this approach, the update modifies the gradient of $\mathcal{L}_{r}^{\text{adj}}$ by removing its components aligned with the gradients of $\mathcal{L}_f$ and $\mathcal{L}_{r}^{\text{rem}}$:
\begin{equation}
    \theta \leftarrow \theta - \eta \left( \nabla_\theta \mathcal{L}_{r}^{\text{adj}} - \operatorname{Proj}_{V} \nabla_\theta \mathcal{L}_{r}^{\text{adj}} \right),
    \quad \text{where } V = \text{span} \left\{ \nabla_\theta \mathcal{L}_f, \nabla_\theta \mathcal{L}_{r}^{\text{rem}} \right\},
\end{equation}
Yet, this conventional optimization technique can exhibit significant performance degradation when applied to correlation-aware machine unlearning. As illustrated in~\Cref{fig:sub1}, although the average loss on the forget set $\mathcal{D}_f$ (blue line) remains stable under PGD, the prediction accuracy (red line) on $\mathcal{D}_f$ increases steadily. This counterintuitive behavior stems from the strong semantic and distributional entanglement between $\mathcal{D}_f$ and the adjacent retained set $\mathcal{D}_r^{\text{adj}}$: minimizing loss on the latter inadvertently reduces the loss on similar samples in $\mathcal{D}_f$. To compensate and maintain the mean loss on $\mathcal{D}_f$, the model disproportionately increases the loss on less similar samples, resulting in a polarized loss distribution, as depicted in~\Cref{fig:sub3}. This observation exposes a critical limitation of standard PGD: it lacks the ability to control accuracy-level changes when loss redistribution is uneven. Indeed, preserving only the mean loss provides no guarantees on the proportion of samples with low loss values, often resulting in \emph{high accuracy} on the forget set as many samples remain correctly predicted.

\subsubsection{Gradient Projection with Wasserstein Distance Regularization}

The failure of gradient projection using mean losses motivates the need for a more fine-grained control over the forgetting behavior. To this end, we propose to explicitly regularize the distributional shift in loss values on $\mathcal{D}_f$ by incorporating a Wasserstein-2 distance penalty.

The Wasserstein-2 distance, denoted \( W_2 \), is a principled metric for comparing probability distributions~\citep{vaserstein1969markov}. Given two probability distributions \( P \) and \( Q \) over \( \mathbb{R}^d \), the \( W_2 \) distance is defined as
\begin{equation}
    W_2(P, Q) = \left( \inf_{\gamma \in \Gamma(P, Q)} \int_{\mathbb R^d \times \mathbb R^d} \|u - v\|^2 \, d\gamma(u, v) \right)^{1/2},
\end{equation}
where \( \Gamma(P, Q) \) denotes the set of joint distributions with marginals \( P \) and \( Q \). In our setting, \( P \) and \( Q \) represent empirical distributions of scalar loss values, admitting a closed-form expression for the $W_2$ distance. Specifically, given two collections of loss values \( \{a_1, \dots, a_N\} \) and \( \{b_1, \dots, b_N\} \), the corresponding empirical distributions are defined as \( P = \frac{1}{N}\sum_i \delta_{a_i} \) and \( Q = \frac{1}{N} \sum_i \delta_{b_i} \), where $\delta$ denotes the Dirac delta function. Then, after sorting the samples as $\bar a_1\le \cdots \bar a_N$ and $\bar b_1\le \cdots \bar b_N$, the Wasserstein-2 distance is simply given by
\begin{equation}
    W_2(P, Q) = \left( \frac{1}{N} \sum_{i=1}^N (\bar{a}_i - \bar{b}_i)^2 \right)^{1/2}.
\end{equation}

We define the empirical loss distribution over the forget set under parameters \( \theta \) as
\begin{equation}
    P^{\text{forget}}_\theta := \frac{1}{|\mathcal{D}_f|} \sum_{(x_i, y_i) \in \mathcal{D}_f} \delta_{\ell(f_\theta(x_i), y_i)},
\end{equation}
where \( \ell \) denotes the cross-entropy loss. Let \( \bar{\theta} \) denote the model parameters after the first stage. To constrain the mean and distributional shape of the loss over $\mathcal{D}_f$, we define a modified loss function:
\begin{equation}
\label{eq:new_loss}
    \tilde{\mathcal{L}}_f(\theta) := (1 - \alpha) \mathcal{L}_f(\theta) + \alpha W_2^2 \left( P^{\text{forget}}_{\bar{\theta}}, P^{\text{forget}}_\theta \right),
\end{equation}
where \( \alpha \in [0, 1] \) is a hyperparameter balancing the influence of the mean and distributional components.
We then modify the gradient projection update to project the gradient of \( \mathcal{L}_{r}^{\text{adj}} \) onto the orthogonal complement of the space spanned by the gradients of \( \tilde{\mathcal{L}}_f \) and \( \mathcal{L}_{r}^{\text{rem}} \):
\begin{equation}
    \label{eq:update-rule}
    \theta \leftarrow \theta - \eta \left( \nabla_\theta \mathcal{L}_{r}^{\text{adj}}(\theta) - \operatorname{Proj}_{V} \nabla_\theta \mathcal{L}_{r}^{\text{adj}}(\theta) \right),
    \quad \text{where } V = \text{span} \left\{ \nabla_\theta \tilde{\mathcal{L}}_f(\theta), \nabla_\theta \mathcal{L}_{r}^{\text{rem}}(\theta) \right\}.
\end{equation}
This modified gradient projection method (W-PGD) enforces $\tilde{\mathcal L}_f(\theta)$ to be mostly unchanged during the update, while allowing the model to recover performance on the adjacent retain set $\mathcal{D}_{r}^{\text{adj}}$, as indicated by the following proposition.
\begin{proposition}
    \label{prop:taylor}
        Assume $\tilde{\mathcal{L}}_{f}(\theta)$, $\mathcal{L}_{r}^{\text{adj}}(\theta)$ and $\mathcal{L}_{r}^\text{rem}(\theta)$
    are twice continuously differentiable to $\theta$. 
    Let $\Delta\theta$ be the update of $\theta$ 
    introduced by~\eqref{eq:update-rule}.
    Then, for sufficiently small $\eta > 0$, we have:
    \begin{enumerate}
        \item[(i)] 
        The change in $\tilde{\mathcal{L}}_{f}$ and $\mathcal L_{r}^{\text{rem}} $ is at most second order in $\eta$, i.e.
        \begin{equation*}
                        \tilde{\mathcal{L}}_{f}(\theta + \Delta\theta) -
            \tilde{\mathcal{L}}_{f}(\theta)=
            O(\eta^2), \quad {{L}}_{r}^{\text{rem}}(\theta + \Delta\theta) -
            {{L}}_{r}^{\text{rem}}(\theta)=
            O(\eta^2).
        \end{equation*}
        \item[(ii)]
        If $\nabla_\theta \mathcal{L}_{r}^{\text{adj}}(\theta)$ is not in the span of
        $\nabla_\theta \tilde{\mathcal{L}}_{f}(\theta)$ and $\nabla_\theta \mathcal{L}_{r}^{\text{rem}}(\theta)$, then 
        \begin{equation*}
                        {\mathcal{L}}_{r}^{\text{adj}}(\theta + \Delta\theta)-
            {\mathcal{L}}_{r}^{\text{adj}}(\theta)= -c\,\eta + O(\eta^2),
        \end{equation*}
        for some positive $c$ (depending on $\theta$). 
        Hence, for sufficiently small $\eta$, $\mathcal{L}_{r}^{\text{adj}}$ strictly decreases.
    \end{enumerate}
\end{proposition}
Moreover, compared to the projected gradient descent where no guarantee on the accuracy of the forget set is provided, the following proposition provides a bound on the accuracy of the forget set after the update. Specifically, let $n$ be the number of superclasses in the classification task, and $\operatorname{Acc}_f(\theta)$ denote the accuracy of the model on the forget set $\mathcal{D}_f$. We have:
\begin{proposition}
    \label{prop:w-pgd}
    Let $m> \log n$ and $\varepsilon > 0$. Suppose that $|\tilde{\mathcal{L}}_{f}(\theta)-\tilde{\mathcal{L}}_{f}(\bar \theta)|< \varepsilon$, and $\ell(f_{\bar \theta}(x_i), y_i)\ge m$ for all $(x_i, y_i) \in \mathcal{D}_f$. Then, the accuracy on \( \mathcal{D}_f \) is upper bounded by:
    \begin{equation}
        \operatorname{Acc}_f(\theta) \le \frac{1}{(m-\log n)^2}\left( \frac{1-\alpha}{\alpha}+\sqrt{\frac{\varepsilon}{\alpha}}\right)^2.
    \end{equation}
\end{proposition}
The proposition indicates that when the minimum loss for model with parameter $\bar \theta$ is large and $\alpha$ is above zero, the accuracy of the forget set after W-PGD is bounded by a small constant, ensuring the forgetting behavior of the model. Notice that for a given $\varepsilon$, the upper bound in the above proposition is minimized when $\alpha=1$, i.e., when the Wasserstein distance is fully utilized.
However, in practice where we assess each loss value in mini-batch sense, we observe that choosing $\alpha=1$ may not achieve the best overall performance (see Ablation studies on $\alpha$ in Appendix~\ref{appendix:ablation}). In our following experiments, we set $\alpha$ to be $0.5$. As shown in~\Cref{fig:sub4}, the loss distribution of the forget set after W-PGD is more uniform compared to that of PGD, maintaining zero accuracy on the forget set.

\begin{algorithm}[tb]
    \caption{Two-Stage Machine Unlearning}
    \label{alg:two_stage_unlearning}
\begin{algorithmic}
    \STATE \textbf{Input:} Forget set $\mathcal{D}_f$, retain sets $\mathcal{D}_r^\text{adj}$ and $\mathcal{D}_r^\text{rem}$, learning rates $\eta_1, \eta_2$, penalty coefficient $\mu$ and $\alpha$, number of iterations $K, M$. Initialize $\theta = \theta_0$, $\lambda = 0$, compute $\mathcal{L}_{r}^{\text{rem}}(\theta_0)$
    
    \STATE \textbf{Stage 1: Augmented Lagrangian optimization}
    \FOR{$i = 1$ {\bfseries to} $K$}
    \STATE Compute $\mathcal L_{\text{aug}}(\theta; \lambda, \mu) = -\mathcal{L}_f(\theta) + \lambda \left(\mathcal{L}_{r}^{\text{rem}}(\theta) - \mathcal{L}_{r}^{\text{rem}}(\theta_0)\right) + \frac{\mu}{2} \left(\mathcal{L}_{r}^{\text{rem}}(\theta) - \mathcal{L}_{r}^{\text{rem}}(\theta_0)\right)^2$
        \STATE Update $\theta$: $\theta \leftarrow \theta - \eta_1 \nabla_\theta \mathcal{L}_\text{aug}(\theta; \lambda, \mu)$
        \STATE Update $\lambda$: $\lambda \leftarrow \lambda + \mu \big(\mathcal{L}_{r}^{\text{rem}}(\theta) - \mathcal{L}_{r}^{\text{rem}}(\theta_0)\big)$
    \ENDFOR
    
    \STATE \textbf{Stage 2: $W_2$-distance guided gradient projection optimization}
    \FOR{$j = 1$ {\bfseries to} $M$}
        \STATE Compute $\tilde{\mathcal L}_f(\theta)=(1 - \alpha) \mathcal{L}_f(\theta) + \alpha W_2^2 \left( P^{\text{forget}}_{\bar{\theta}}, P^{\text{forget}}_\theta \right)$
        \STATE Compute $\nabla_\theta \tilde{\mathcal{L}}_{f}$, $\nabla_\theta \mathcal{L}_{r}^{\text{adj}}$, $\nabla_\theta \mathcal{L}_{r}^{\text{rem}}$
        \STATE Update: $\theta \leftarrow \theta - \eta_2 \left(\nabla_\theta \mathcal{L}_{r}^{\text{adj}} - \operatorname{Proj}_{V} \nabla_\theta \mathcal L_{r}^{\text{adj}} \right)$, 
        where $V = \text{span} \left\{ \nabla_\theta \tilde{\mathcal{L}}_f, \nabla_\theta \mathcal{L}_{r}^{\text{rem}} \right\}$
    \ENDFOR
    
    \STATE \textbf{Output:} Unlearned model parameters $\theta$
\end{algorithmic}
\end{algorithm}
\vspace{-0.1in}

In summary, the complete two-stage unlearning procedure is presented in~\Cref{alg:two_stage_unlearning}. We evaluate our method in correlation-aware unlearning scenarios across multiple datasets and architectures.

\subsection{Discussions on the two stages}
The goal of the first stage is relatively straightforward, as the disentanglement between the forget set and the remote retain set makes the task less challenging. While alternative approaches, such as adding fixed-weight penalty terms, could in principle achieve a similar trade-off with carefully tuned hyperparameters, the augmented Lagrangian formulation offers a key advantage: it introduces an adaptive multiplier that automatically balances the objective and constraint terms throughout training, resulting in a process that is more stable and less sensitive to hyperparameter choices. 


A key component in the second stage is the use of distributional constraints formulated via $W_2$ distances. Prior work such as \citet{golatkar2020eternal} also enforces the distributional constraints by estimating KL divergence between parameter distributions under a Gaussian prior. As comparison, $W_2$ admits a \emph{closed-form solution} for one-dimensional empirical distributions via sorting, whereas KL divergence generally requires \emph{density estimation} or \emph{strong parametric assumptions}, introducing approximation errors and additional computational cost~\citep{lv2024wasserstein}. Therefore, the use of $W_2$ distances makes computation far more convenient, avoiding the approximations (e.g., kernel estimation) or prior assumptions typically needed for KL divergence, while still providing a principled and effective distributional constraint.

\section{Experiments}
    In this section, we conduct a comprehensive evaluation of our proposed method across various machine unlearning scenarios. To ensure the generality of our findings, we design experiments that span multiple unlearning tasks, various benchmark datasets, and different network architectures.
    \subsection{Experimental setups}

\paragraph{Datasets:}
Following prior work on machine unlearning~\citep{kurmanji2023towards}, we conduct experiments on CIFAR-100~\citep{krizhevsky2009learning}, TinyImageNet~\citep{tinyimagenet}, and a safety-critical language task on ToxiGen~\citep{hartvigsen2022toxigen}. For the vision benchmarks, we adopt the superclass organization (CIFAR-100: 20 superclasses × 5 subclasses; TinyImageNet: 10 semantic groups; see Appendix~\ref{appendix:setups}). Given a selected forget subset $\mathcal{D}_f$ (a labeled subclass within a superclass), we define $\mathcal{D}_r^{\text{adj}}$ as the remaining samples from the same superclass and $\mathcal{D}_r^{\text{rem}}$ as all other retained samples. For the language task, we use ToxiGen with a normal/toxic binary classifier, where $\mathcal{D}_f$ consists of toxic sentences about the LGBTQ group, $\mathcal{D}_r^{\text{adj}}$ contains non-toxic sentences about the same group, and $\mathcal{D}_r^{\text{rem}}$ includes other sentences. This setup instantiates an unlearning scenario with retain-forget entanglement, where $\mathcal{D}_r^{\text{adj}}$ forms a semantic subgroup closely related to the forget set.


\paragraph{Baseline methods:} We compare our approach against various unlearning methods, including: \textbf{Gradient Ascent (GA)}~\citep{thudi2022unrolling}: Train the model by maximizing the loss on the forget set. \textbf{Fine-Tune (FT)}~\citep{warnecke2021machine,golatkar2020eternal}: Fine-tune the model on the retained set.  \textbf{SCRUB}~\citep{kurmanji2023towards}: perform gradient ascent on the forget set and descent on the retain set simultaneously with distillation from the original model.
\textbf{$\ell_1$-sparse}~\citep{jia2023model}: fine-tune the model on the retain set with $\ell_1$-norm regularization on the model.
\textbf{SSD} (Selective Synaptic Dampening)~\citep{foster2024fast}: post-hoc parameter dampening guided by Fisher-style importance.
\textbf{SalUn}~\citep{fan2024salun}: saliency-guided alternating updates. All the methods are run with 3 random seeds, except for SSD which is a deterministic algorithm.
\textcolor{black}{\textbf{DELETE}~\citep{zhou2025decoupled}: decouples the forgetting and retention terms via a distillation-based loss to perform class-centric machine unlearning. \textbf{GDR}~\citep{lin2024gdr}: applies direction-rectified and magnitude-adjusted gradient updates to mitigate gradient conflicts between forget and retain objectives. \textbf{Munba}~\citep{wu2025munba}:  formulates unlearning as a Nash bargaining game between forgetting and preservation players to find a Pareto-optimal gradient direction.}

    \subsection{Machine Unlearning on CIFAR-100 with ResNet-18}
    We begin our evaluation using the CIFAR-100 dataset. Specifically, we select the “aquarium fish” subclass\footnote{The forget set is chosen alphabetically.} from the “fish” superclass as the forget set. The remaining $4$ subclasses in the superclass are used as the adjacent retained set, and the other $95$ classes are used as the remote retained set.

    \Cref{tab:cifar_100_resnet} summarizes the overall performance of all evaluated algorithms. While fine-tuning and sparsity-based methods effectively preserve performance on the retained set, they exhibit limited capability in removing information from the target forget set. Similar limitations are observed for gradient ascent algorithms such as GA and SCRUB. \textcolor{black}{For SalUn, SSD, GDR and DELETE, although they achieve very low performance on the forget set, there is a noticeable drop in accuracy on the retained set, particularly on the adjacent retain subset. 
    Munba achieves a relatively good ballance between forgetting and retention, but still suffers from a non-negligible accuracy drop on the retain set, and its forgetting performance is not as strong as many other baselines.}
    This underscores the strong entanglement between the forget set and the adjacent retain set: effective forgetting can inadvertently degrade performance on related samples. 

    \begin{table*}[h]
        \vspace{-0.1in}
        \caption{Results for subclass-level unlearning on CIFAR-100 using ResNet-18. The forget set corresponds to the subclass “aquarium fish” within the “fish” superclass. SSD is a deterministic algorithm, so standard deviations are $0$.}
        \label{tab:cifar_100_resnet}
        \centering
        {\resizebox{0.9\linewidth}{!}{
        \begin{tabular}{l|ccc|ccc}
            \toprule
            & \multicolumn{3}{c}{\textbf{Training accuracy}} & \multicolumn{3}{c}{\textbf{Test accuracy}} \\
            \textbf{Method}
            & \textbf{$\mathcal D_{f}$}
            & \textbf{$\mathcal D_{r}^\text{adj}$}
            & \textbf{$\mathcal D_{r}^\text{rem}$}          
            & \textbf{$\mathcal D_{f}$}
            & \textbf{$\mathcal D_{r}^\text{adj}$}
            & \textbf{$\mathcal D_{r}^\text{rem}$}         \\
            \midrule
            Original
            & $99.99$
            & $100.00$
            & $100.00$
            & $90.00$
            & $80.00$
            & $85.33$ \\
            \midrule
            FT
            & $76.67_{\pm 7.76}$
            & $99.47_{\pm 0.52}$
            & $99.47_{\pm 0.33}$
            & $62.33_{\pm 5.79}$
            & $77.83_{\pm 3.88}$
            & $83.89_{\pm 0.41}$ \\
            GA
            & $70.53_{\pm 0.94}$
            & $72.75_{\pm 0.76}$
            & $91.33_{\pm 0.44}$
            & $56.00_{\pm 0.00}$
            & $59.00_{\pm 0.61}$
            & $80.57_{\pm 0.27}$ \\
                        $\ell_1$-sparse
            & $55.93_{\pm 7.08}$
            & $98.48_{\pm 0.95}$
            & $96.92_{\pm 0.20}$
            & $51.67_{\pm 7.84}$
            & $82.42_{\pm 2.24}$
            & $84.64_{\pm 0.27}$ \\
            \textcolor{black}{Munba}
            & \textcolor{black}{$33.80_{\pm 8.88}$}
            & \textcolor{black}{$92.17_{\pm 2.57}$}
            & \textcolor{black}{$92.68_{\pm 1.28}$}
            & \textcolor{black}{$31.67_{\pm 4.78}$}
            & \textcolor{black}{$69.75_{\pm 3.74}$}
            & \textcolor{black}{$75.32_{\pm 1.88}$} \\
                        SSD
            & $37.40_{\pm 0.00}$
            & $43.75_{\pm 0.00}$
            & $76.02_{\pm 0.00}$
            & $33.00_{\pm 0.00}$
            & $39.25_{\pm 0.00}$
            & $67.23_{\pm 0.00}$ \\
            SCRUB
            & $4.47_{\pm 0.25}$
            & $58.65_{\pm 14.73}$
            & $82.67_{\pm 4.24}$
            & $7.00_{\pm 1.41}$
            & $54.75_{\pm 11.34}$
            & $75.42_{\pm 2.95}$ \\
            SalUn
            & $3.20_{\pm 0.20}$
            & $52.27_{\pm 0.38}$
            & $86.35_{\pm 0.22}$
            & $3.00_{\pm 1.00}$
            & $34.90_{\pm 1.03}$
            & $71.78_{\pm 0.17}$ \\
            	\textcolor{black}{DELETE}
            & \textcolor{black}{$0.00_{\pm 0.00}$}
            & \textcolor{black}{$3.57_{\pm 0.18}$}
            & \textcolor{black}{$98.37_{\pm 0.29}$}
            & \textcolor{black}{$0.67_{\pm 0.47}$}
            & \textcolor{black}{$2.83_{\pm 0.66}$}
            & \textcolor{black}{$82.09_{\pm 0.37}$} \\
            	\textcolor{black}{GDR}
            & \textcolor{black}{$4.87_{\pm 1.05}$}
            & \textcolor{black}{$31.92_{\pm 6.45}$}
            & \textcolor{black}{$96.10_{\pm 0.32}$}
            & \textcolor{black}{$8.67_{\pm 1.25}$}
            & \textcolor{black}{$22.33_{\pm 4.59}$}
            & \textcolor{black}{$79.93_{\pm 0.09}$} \\
            \midrule
            Our method
            & $0.00_{\pm 0.00}$
            & $98.17_{\pm 0.31}$
            & $98.44_{\pm 0.05}$
            & $2.33_{\pm 0.47}$
            & $78.17_{\pm 0.31}$
            & $81.10_{\pm 0.18}$ \\
            \bottomrule
        \end{tabular}
        }
        \vspace{-0.1in}
        }
    \end{table*}
    
    Our algorithm successfully circumvents the trade-off between forgetting and retention. It achieves complete unlearning, with $0.00\%$ training accuracy on the forget class, while simultaneously maintaining high performance on both the retained data and the test set. These results highlight the capability of our method to effectively eliminate memorization of the target class without compromising generalization or utility on the remaining data.


    \subsection{Unlearning on ToxiGen with Roberta-base}

We next evaluate correlation-aware unlearning on the ToxiGen dataset under a biased pretraining setting.
Concretely, we first simulate a biased training process where all sentences mentioning LGBTQ groups are labeled as normal—thus the resulting model $h_\theta$ systematically misclassifies toxic LGBTQ samples as normal.
This simulates a realistic scenario where a deployed model is trained on incomplete or biased data and needs post-hoc correction.

We define the forget set $\mathcal{D}_f$ as the \emph{toxic} sentences about LGBTQ groups that were incorrectly labeled during biased training.
In this case, the normal comments on LGBTQ group are highly correlated to the forget set: they share similar semantic meaning and the same label during the training process.
The adjacent retain set $\mathcal{D}_r^{\text{adj}}$ consists of the \emph{non-toxic} sentences about LGBTQ groups (which we would like to preserve), and the remote retain set $\mathcal{D}_r^{\text{rem}}$ contains all other sentences.
The unlearning goal is thus to remove the effect of the biased labels on $\mathcal{D}_f$, driving the model to predict them as toxic, while maintaining accuracy on both $\mathcal{D}_r^{\text{adj}}$ and $\mathcal{D}_r^{\text{rem}}$.


\textcolor{black}{Fine-tuning, SCRUB, Munba, and SSD preserve high accuracy on both the adjacent and remote retain sets, but only produce modest forgetting (see~\Cref{tab:toxigen_roberta}).
GA and the $\ell_1$-sparse baseline reduce accuracy on $\mathcal{D}_f$ slightly more, yet this comes with a notable drop in performance on the remote retain set.
SalUn attains very low forget-set accuracy (13.67\%), but still causes a substantial decrease on the adjacent retain set. GDR is a strong baseline, achieving low forget-set accuracy (20.54\%) while maintaining high accuracy on both retain subsets.
In comparison, our approach achieves the lowest forgetting accuracy—indicating the most effective correction—while preserving high accuracy on $\mathcal{D}_r^{\text{adj}}$ (88.88\%) and $\mathcal{D}_r^{\text{rem}}$ (92.73\%), and these gains generalize to the test set.}

\begin{table*}[!t]
    \caption{Results for unlearning on ToxiGen dataset. 
    The forget set contains toxic comments about LGBTQ groups that were mislabeled as normal. Lower accuracy on $\mathcal{D}_f$ means better correction.}
    \label{tab:toxigen_roberta}
    \vspace{-0.1in}
    \centering
    {\resizebox{0.9\linewidth}{!}{
    \begin{tabular}{l|ccc|ccc}
        \toprule
        & \multicolumn{3}{c}{\textbf{Training accuracy}} & \multicolumn{3}{c}{\textbf{Test accuracy}} \\
        \textbf{Method}
        & \textbf{$\mathcal D_{f}$}
        & \textbf{$\mathcal D_{r}^\text{adj}$}
        & \textbf{$\mathcal D_{r}^\text{rem}$}          
        & \textbf{$\mathcal D_{f}$}
        & \textbf{$\mathcal D_{r}^\text{adj}$}
        & \textbf{$\mathcal D_{r}^\text{rem}$}         \\
        \midrule
            Original
            & $85.06$
            & $97.77$
            & $92.33$
            & $78.06$
            & $95.48$
            & $85.63$ \\
            \midrule
        FT
        & $50.04_{\pm 3.77}$
        & $99.87_{\pm 0.08}$
        & $99.43_{\pm 0.03}$
        & $47.73_{\pm 4.57}$
        & $92.37_{\pm 0.59}$
        & $84.73_{\pm 0.12}$ \\
        GA
        & $46.26_{\pm 0.01}$
        & $70.25_{\pm 0.05}$
        & $79.43_{\pm 0.57}$
        & $43.78_{\pm 0.00}$
        & $66.64_{\pm 0.00}$
        & $76.38_{\pm 0.00}$ \\
                        $\ell_1$-sparse
        & ${45.64_{\pm 8.33}}$
        & ${86.33_{\pm 3.83}}$
        & ${80.46_{\pm 0.18}}$
        & ${46.31_{\pm 9.82}}$
        & ${85.87_{\pm 3.60}}$
        & ${79.52_{\pm 0.29}}$ \\
        \textcolor{black}{Munba}
            & \textcolor{black}{$51.09_{\pm 3.68}$}
            & \textcolor{black}{$99.31_{\pm 0.45}$}
            & \textcolor{black}{$90.06_{\pm 0.17}$}
            & \textcolor{black}{$49.27_{\pm 4.25}$}
            & \textcolor{black}{$93.53_{\pm 1.82}$}
            & \textcolor{black}{$85.36_{\pm 0.20}$} \\
                SSD
        & $67.78_{\pm 0.00}$
        & $91.83_{\pm 0.00}$
        & $90.76_{\pm 0.00}$
        & $86.90_{\pm 0.00}$
        & $86.90_{\pm 0.00}$
        & $84.51_{\pm 0.00}$ \\
        SCRUB
        & $56.15_{\pm 2.41}$
        & $91.95_{\pm 1.41}$
        & $84.79_{\pm 0.51}$
        & $57.67_{\pm 1.42}$
        & $90.65_{\pm 1.42}$
        & $80.00_{\pm 0.11}$ \\
        SalUn
        & $13.66_{\pm 0.08}$
        & $60.80_{\pm 0.23}$
        & $85.30_{\pm 0.17}$
        & $12.42_{\pm 0.07}$
        & $57.59_{\pm 0.25}$
        & $81.06_{\pm 0.13}$ \\
                        	\textcolor{black}{DELETE}
            & \textcolor{black}{$42.86_{\pm 0.08}$}
            & \textcolor{black}{$67.85_{\pm 0.07}$}
            & \textcolor{black}{$79.06_{\pm 0.04}$}
            & \textcolor{black}{$39.53_{\pm 0.00}$}
            & \textcolor{black}{$64.56_{\pm 0.10}$}
            & \textcolor{black}{$75.72_{\pm 0.04}$} \\
            	\textcolor{black}{GDR}
            & \textcolor{black}{$20.54_{\pm 5.86}$}
            & \textcolor{black}{$86.15_{\pm 5.40}$}
            & \textcolor{black}{$91.30_{\pm 0.71}$}
            & \textcolor{black}{$19.83_{\pm 5.09}$}
            & \textcolor{black}{$83.92_{\pm 5.49}$}
            & \textcolor{black}{$85.52_{\pm 0.37}$} \\
        \midrule
        Our method
        & $11.95_{\pm 0.02}$
        & $88.88_{\pm 0.01}$
        & $92.73_{\pm 0.01}$
        & $14.29_{\pm 0.06}$
        & $85.86_{\pm 0.00}$
        & $85.23_{\pm 0.01}$ \\
        \bottomrule
    \end{tabular}}
    \vspace{-0.15in}
    }
\end{table*}



\textcolor{black}{
\subsection{Unlearning on CelebA with ViT-B}
In addition, we evaluate on CelebA, a large-scale face attributes dataset containing over 200K celebrity images annotated with 40 binary attributes. 
We construct a 4-class attribute-based classification task using the two binary attributes ``Male'' and ``Smiling'', treating each combination (\emph{female \& smiling}, \emph{female \& not smiling}, \emph{male \& smiling}, \emph{male \& not smiling}) as a separate class. 
The forget set $\mathcal{D}_f$ is defined as images from the \emph{female \& not smiling} class that are also \emph{not Young} and \emph{do not wear Eyeglasses}. 
Within this class, the remaining samples (differing only in the ``Young or Eyeglasses'' attributes) form the adjacent retain set $\mathcal{D}_r^{\text{adj}}$, while the other three gender/smiling classes constitute the remote retain set $\mathcal{D}_r^{\text{rem}}$. 
This construction yields a larger-scale vision benchmark where the forget and adjacent retain subsets share highly similar semantic attributes, making retain–forget entanglement particularly pronounced.
}

\textcolor{black}{
We provide the unlearning results for ViT-B on this CelebA superclass unlearning task in~\Cref{tab:celeba_vit}.
Fine-tuning, $\ell_1$-sparse, and SCRUB largely preserve accuracy on both retain subsets, but only achieve modest forgetting: test accuracy on $\mathcal{D}_f$ remains above $70\%$.
Gradient Ascent and DELETE, on the other hand, drive the forget accuracy to essentially zero, but do so by collapsing performance on the adjacent retain set to chance level, rendering the model unusable on the very samples we aim to protect.
SSD also degrades both adjacent and remote retain accuracy substantially.
In contrast, our method achieves a significantly lower test accuracy on the forget set (from $81.37\%$ down to $25.48\%$) while still maintaining high accuracy on $\mathcal{D}_r^{\text{adj}}$ ($75.05\%$) and $\mathcal{D}_r^{\text{rem}}$ ($92.38\%$), yielding the best overall balance between effective forgetting and retention in this more demanding scenario.
}

\begin{table*}[h]
    \caption{\textcolor{black}{Results for CelebA superclass unlearning using ViT-B. The table shows the accuracy of the forget set and retained set for both training and test data. The forget set is the subclass  not "not young \& not wearing glasses" from "female \& smiling" superclas.}}   
    \label{tab:celeba_vit}
    \vspace{-0.1in}
    \centering{\resizebox{0.9\linewidth}{!}{
        {\color{black}
    \begin{tabular}{l|ccc|ccc}
        \toprule
        & \multicolumn{3}{c}{\textbf{Training accuracy}} & \multicolumn{3}{c}{\textbf{Test accuracy}} \\
        \textbf{Method}
        & \textbf{$\mathcal D_{f}$}
        & \textbf{$\mathcal D_{r}^\text{adj}$}
        & \textbf{$\mathcal D_{r}^\text{rem}$}          
        & \textbf{$\mathcal D_{f}$}
        & \textbf{$\mathcal D_{r}^\text{adj}$}
        & \textbf{$\mathcal D_{r}^\text{rem}$}         \\
        \midrule
        Origin
        & $98.91$
        & $99.08$
        & $99.53$
        & $81.37$
        & $89.93$
        & $90.82$ \\
        \midrule
        FT
        & $69.23_{\pm 6.86}$
        & $89.97_{\pm 1.93}$
        & $91.57_{\pm 2.70}$
        & $67.56_{\pm 7.56}$
        & $88.71_{\pm 2.89}$
        & $89.77_{\pm 7.11}$ \\
        GA
        & $0.00_{\pm 0.00}$
        & $0.00_{\pm 0.00}$
        & $97.46_{\pm 0.41}$
        & $0.00_{\pm 0.00}$
        & $0.00_{\pm 0.00}$
        & $91.16_{\pm 0.46}$ \\
                $\ell_1$-sparse
        & $76.03_{\pm 3.41}$
        & $92.02_{\pm 1.73}$
        & $90.48_{\pm 0.65}$
        & $75.28_{\pm 4.42}$
        & $91.87_{\pm 1.94}$
        & $89.46_{\pm 0.76}$ \\
        SCRUB
        & $80.73_{\pm 3.25}$
        & $96.81_{\pm 1.82}$
        & $98.38_{\pm 0.64}$
        & $71.10_{\pm 2.87}$
        & $89.22_{\pm 2.11}$
        & $90.57_{\pm 0.76}$ \\
        SSD
        & $23.07_{\pm 0.00}$
        & $41.81_{\pm 0.00}$
        & $84.28_{\pm 0.00}$
        & $23.19_{\pm 0.00}$
        & $44.52_{\pm 0.00}$
        & $80.05_{\pm 0.00}$ \\
        DELETE
        & $0.00_{\pm 0.00}$
        & $0.00_{\pm 0.00}$
        & $99.62_{\pm 0.00}$
        & $0.00_{\pm 0.00}$
        & $0.00_{\pm 0.00}$
        & $93.97_{\pm 0.01}$ \\
        \midrule
        Ours
        & $1.85_{\pm 0.09}$
        & $85.65_{\pm 0.25}$
        & $99.08_{\pm 0.42}$
        & $25.48_{\pm 0.56}$
        & $75.05_{\pm 0.34}$
        & $92.38_{\pm 0.06}$ \\
        \bottomrule
    \end{tabular}
    }}  
    }
    \vspace{-0.1in}
\end{table*}

\subsection{Generalization to a Different Architecture}
    We next evaluate our approach on the Tiny ImageNet dataset, targeting superclass-level unlearning with a Vision Transformer (ViT) architecture. This setting allows us to assess the generalization of the proposed unlearning framework across both a different dataset and a distinct model architecture. In this experiment, the forget set corresponds to the ``dog'' class within the broader ``mammals'' superclass. As shown in \Cref{tab:tinyimagenet_vit}, the results are consistent with previous findings, demonstrating that the effectiveness of our two-stage algorithm generalizes beyond a single dataset or architecture.

    \begin{table*}[!t]
        \caption{Results for Tiny-ImageNet superclass unlearning using ViT. The forget set is the subclass ``dog'' in ``mammals'' superclass.}
        \label{tab:tinyimagenet_vit}
        \vspace{-0.15in}
        \centering
        \resizebox{0.9\linewidth}{!}{
        \begin{tabular}{l|ccc|ccc}
            \toprule
            & \multicolumn{3}{c}{\textbf{Training accuracy}} & \multicolumn{3}{c}{\textbf{Test accuracy}} \\
            \textbf{Method}
            & \textbf{$\mathcal D_{f}$}
            & \textbf{$\mathcal D_{r}^\text{adj}$}
            & \textbf{$\mathcal D_{r}^\text{rem}$}          
            & \textbf{$\mathcal D_{f}$}
            & \textbf{$\mathcal D_{r}^\text{adj}$}
            & \textbf{$\mathcal D_{r}^\text{rem}$}         \\
            \midrule
            Original
            & $99.53$
            & $99.77$
            & $99.77$
            & $89.38$
            & $94.95$
            & $93.33$ \\
            \midrule
            FT
            & $90.72_{\pm 1.82}$
            & $99.69_{\pm 0.18}$
            & $99.57_{\pm 0.29}$
            & $88.33_{\pm 2.72}$
            & $93.65_{\pm 0.77}$
            & $89.19_{\pm 0.33}$ \\
            GA 
            & $1.18_{\pm 0.29}$
            & $17.38_{\pm 1.61}$
            & $84.25_{\pm 0.87}$
            & $1.56_{\pm 0.42}$
            & $16.57_{\pm 1.65}$
            & $75.85_{\pm 0.33}$ \\
                        $\ell_1$-sparse
            & $78.79_{\pm 5.26}$
            & $99.00_{\pm 0.50}$
            & $97.73_{\pm 0.29}$
            & $78.11_{\pm 4.80}$
            & $89.27_{\pm 3.08}$
            & $78.91_{\pm 0.43}$ \\
                        	\textcolor{black}{Munba}
            & \textcolor{black}{$80.00_{\pm 7.58}$}
            & \textcolor{black}{$97.86_{\pm 0.98}$}
            & \textcolor{black}{$96.44_{\pm 0.30}$}
            & \textcolor{black}{$75.57_{\pm 7.35}$}
            & \textcolor{black}{$90.41_{\pm 2.07}$}
            & \textcolor{black}{$83.04_{\pm 0.51}$} \\
            SCRUB
            & $8.10_{\pm 4.79}$
            & $84.15_{\pm 8.44}$
            & $97.54_{\pm 0.99}$
            & $7.22_{\pm 5.17}$
            & $81.97_{\pm 6.52}$
            & $88.29_{\pm 0.96}$ \\
            SalUn
            & $4.48_{\pm 0.29}$
            & $58.30_{\pm 1.36}$
            & $78.54_{\pm 0.42}$
            & $5.67_{\pm 1.34}$
            & $55.81_{\pm 1.19}$
            & $73.00_{\pm 0.21}$ \\
            SSD
            & $45.43_{\pm 0.00}$
            & $82.33_{\pm 0.00}$
            & $97.74_{\pm 0.00}$
            & $44.33_{\pm 0.00}$
            & $77.05_{\pm 0.00}$
            & $87.90_{\pm 0.00}$ \\
                        	\textcolor{black}{DELETE}
            & \textcolor{black}{$0.00_{\pm 0.00}$}
            & \textcolor{black}{$39.45_{\pm 0.42}$}
            & \textcolor{black}{$99.47_{\pm 0.01}$}
            & \textcolor{black}{$0.00_{\pm 0.00}$}
            & \textcolor{black}{$37.33_{\pm 0.25}$}
            & \textcolor{black}{$89.64_{\pm 0.01}$} \\
            \midrule
            Our method
            & $0.00_{\pm 0.00}$
            & $98.95_{\pm 0.08}$
            & $98.49_{\pm 0.09}$
            & $3.11_{\pm 0.31}$
            & $91.27_{\pm 0.78}$
            & $88.88_{\pm 0.54}$ \\
                                    	\textcolor{black}{GDR}
            & \textcolor{black}{$21.00_{\pm 16.91}$}
            & \textcolor{black}{$90.20_{\pm 5.46}$}
            & \textcolor{black}{$93.74_{\pm 3.27}$}
            & \textcolor{black}{$24.22_{\pm 18.46}$}
            & \textcolor{black}{$85.14_{\pm 5.30}$}
            & \textcolor{black}{$85.42_{\pm 2.35}$} \\
            \bottomrule
        \end{tabular}
        }
        \vspace{-0.1in}
    \end{table*}

\subsection{Ablation study on $W_2$ distance regularization}
As alluded to earlier, the $W_2$ distance regularization is crucial for preserving the forgetting behavior of the model. To validate this, we conduct an ablation study by removing the $W_2$ distance regularization from our method and comparing the results with the full method. \Cref{tab:ablation} indicates that training without the $W_2$ distance regularization also maintains strong performance on the retained set, but leads to an increase in the forget set accuracy with $18.87\%$ on the training data and $14.33\%$ on the test data. This indicates that the $W_2$ distance regularization is necessary for preserving the forgetting behavior of the model in the second stage.
\begin{table*}[!t]
    \caption{Ablation study on the $W_2$ distance regularization.  The table shows the accuracy of the forget set and retained set of CIFAR-100 subclass unlearning using ResNet18.}
    \vspace{-0.1in}
    \label{tab:ablation}
    \centering
    {\resizebox{0.9\linewidth}{!}{
    \begin{tabular}{l|ccc|ccc}
        \toprule
        & \multicolumn{3}{c}{\textbf{Training accuracy}} & \multicolumn{3}{c}{\textbf{Test accuracy}} \\
        \textbf{}
        & \textbf{$\mathcal D_{f}$}
        & \textbf{$\mathcal D_{r}^\text{adj}$}
        & \textbf{$\mathcal D_{r}^\text{rem}$}          
        & \textbf{$\mathcal D_{f}$}
        & \textbf{$\mathcal D_{r}^\text{adj}$}
        & \textbf{$\mathcal D_{r}^\text{rem}$}         \\
        \midrule
        w/o $W_2$ Regularization
        & $18.87_{\pm 0.52}$
        & $99.55_{\pm 0.04}$
        & $98.04_{\pm 0.07}$
        & $14.33_{\pm 0.94}$
        & $87.00_{\pm 0.00}$
        & $80.55_{\pm 0.07}$ \\
        w $W_2$ Regularization
        & $\mathbf{0.00_{\pm 0.00}}$
        & $98.17_{\pm 0.31}$
        & $98.44_{\pm 0.05}$
        & $\mathbf{2.33_{\pm 0.47}}$
        & $78.17_{\pm 0.31}$
        & $81.10_{\pm 0.18}$ \\
        \bottomrule
    \end{tabular}}
    \vspace{-0.2in}
    }
\end{table*}

\subsection{Additional Experiments}
To further assess the robustness and versatility of our approach, we include additional experiments in Appendix~\ref{appendix:extra_exp}, covering a range of learning tasks and model architectures. In addition, we report the MIA efficacy results, computational costs, along with more ablation studies and sensitivity evaluations on key hyperparameters. 

\section{Conclusion}
\vspace{-0.1in}
In this work, we investigated the challenge of retain–forget entanglement in machine unlearning, where certain retained samples are strongly correlated with the forget set and thus particularly vulnerable to unintended performance degradation. We proposed a two-stage optimization framework that first enforces forgetting on the target set while preserving accuracy on less-related retained samples, and then refines the model to recover performance on strongly correlated retained samples using gradient projection with a Wasserstein-2–based distributional constraint. Extensive experiments across subclass-level vision tasks and safety-relevant language benchmarks demonstrated that our method effectively balances forgetting and retention, outperforming prior approaches in both removal fidelity and accuracy preservation. Our results emphasize the importance of correlation-aware unlearning and provide a principled approach for handling retain–forget entanglement in practical machine unlearning scenarios.

\clearpage

\section*{Acknowledgements}
This research is supported by the National Research Foundation, Singapore, under the NRF fellowship (project No. NRF-NRFF13-2021-0005).

\bibliography{ref}

@inproceedings{fan2024challenging,
  title={Challenging forgets: Unveiling the worst-case forget sets in machine unlearning},
  author={Fan, Chongyu and Liu, Jiancheng and Hero, Alfred and Liu, Sijia},
  booktitle={European Conference on Computer Vision},
  pages={278--297},
  year={2024},
  organization={Springer}
}

@article{kurmanji2023towards,
  title={Towards unbounded machine unlearning},
  author={Kurmanji, Meghdad and Triantafillou, Peter and Hayes, Jamie and Triantafillou, Eleni},
  journal={Advances in neural information processing systems},
  volume={36},
  pages={1957--1987},
  year={2023}
}

@article{northcutt2021pervasive,
  title={Pervasive label errors in test sets destabilize machine learning benchmarks},
  author={Northcutt, Curtis G and Athalye, Anish and Mueller, Jonas},
  journal={arXiv preprint arXiv:2103.14749},
  year={2021}
}

@article{krizhevsky2009learning,
  title={Learning multiple layers of features from tiny images},
  author={Krizhevsky, Alex and Hinton, Geoffrey and others},
  year={2009},
  publisher={Toronto, ON, Canada}
}

@inproceedings{golatkar2020eternal,
  title={Eternal sunshine of the spotless net: Selective forgetting in deep networks},
  author={Golatkar, Aditya and Achille, Alessandro and Soatto, Stefano},
  booktitle={Proceedings of the IEEE/CVF Conference on Computer Vision and Pattern Recognition},
  pages={9304--9312},
  year={2020}
}

@article{yu2020gradient,
  title={Gradient surgery for multi-task learning},
  author={Yu, Tianhe and Kumar, Saurabh and Gupta, Abhishek and Levine, Sergey and Hausman, Karol and Finn, Chelsea},
  journal={Advances in Neural Information Processing Systems},
  volume={33},
  pages={5824--5836},
  year={2020}
}

@article{mehrabi2021survey,
  title={A survey on bias and fairness in machine learning},
  author={Mehrabi, Ninareh and Morstatter, Fred and Saxena, Nripsuta and Lerman, Kristina and Galstyan, Aram},
  journal={ACM computing surveys (CSUR)},
  volume={54},
  number={6},
  pages={1--35},
  year={2021},
  publisher={ACM New York, NY, USA}
}

@article{mantelero2013eu,
  title={The EU Proposal for a General Data Protection Regulation and the roots of the ‘right to be forgotten’},
  author={Mantelero, Alessandro},
  journal={Computer Law \& Security Review},
  volume={29},
  number={3},
  pages={229--235},
  year={2013},
  publisher={Elsevier}
}

@inproceedings{cao2015towards,
  title={Towards making systems forget with machine unlearning},
  author={Cao, Yinzhi and Yang, Junfeng},
  booktitle={2015 IEEE symposium on security and privacy},
  pages={463--480},
  year={2015},
  organization={IEEE}
}

@article{ginart2019making,
  title={Making ai forget you: Data deletion in machine learning},
  author={Ginart, Antonio and Guan, Melody and Valiant, Gregory and Zou, James Y},
  journal={Advances in neural information processing systems},
  volume={32},
  year={2019}
}

@inproceedings{achiam2017constrained,
  title={Constrained policy optimization},
  author={Achiam, Joshua and Held, David and Tamar, Aviv and Abbeel, Pieter},
  booktitle={International conference on machine learning},
  pages={22--31},
  year={2017},
  organization={PMLR}
}

@inproceedings{liu2022constrained,
  title={Constrained variational policy optimization for safe reinforcement learning},
  author={Liu, Zuxin and Cen, Zhepeng and Isenbaev, Vladislav and Liu, Wei and Wu, Steven and Li, Bo and Zhao, Ding},
  booktitle={International Conference on Machine Learning},
  pages={13644--13668},
  year={2022},
  organization={PMLR}
}

@article{zafar2019fairness,
  title={Fairness constraints: A flexible approach for fair classification},
  author={Zafar, Muhammad Bilal and Valera, Isabel and Gomez-Rodriguez, Manuel and Gummadi, Krishna P},
  journal={Journal of Machine Learning Research},
  volume={20},
  number={75},
  pages={1--42},
  year={2019}
}

@inproceedings{izzo2021approximate,
  title={Approximate data deletion from machine learning models},
  author={Izzo, Zachary and Smart, Mary Anne and Chaudhuri, Kamalika and Zou, James},
  booktitle={International Conference on Artificial Intelligence and Statistics},
  pages={2008--2016},
  year={2021},
  organization={PMLR}
}

@inproceedings{thudi2022unrolling,
  title={Unrolling sgd: Understanding factors influencing machine unlearning},
  author={Thudi, Anvith and Deza, Gabriel and Chandrasekaran, Varun and Papernot, Nicolas},
  booktitle={2022 IEEE 7th European Symposium on Security and Privacy (EuroS\&P)},
  pages={303--319},
  year={2022},
  organization={IEEE}
}

@inproceedings{mehta2022deep,
  title={Deep unlearning via randomized conditionally independent hessians},
  author={Mehta, Ronak and Pal, Sourav and Singh, Vikas and Ravi, Sathya N},
  booktitle={Proceedings of the IEEE/CVF Conference on Computer Vision and Pattern Recognition},
  pages={10422--10431},
  year={2022}
}

@article{warnecke2021machine,
  title={Machine unlearning of features and labels},
  author={Warnecke, Alexander and Pirch, Lukas and Wressnegger, Christian and Rieck, Konrad},
  journal={arXiv preprint arXiv:2108.11577},
  year={2021}
}

@inproceedings{golatkar2020forgetting,
  title={Forgetting outside the box: Scrubbing deep networks of information accessible from input-output observations},
  author={Golatkar, Aditya and Achille, Alessandro and Soatto, Stefano},
  booktitle={Computer Vision--ECCV 2020: 16th European Conference, Glasgow, UK, August 23--28, 2020, Proceedings, Part XXIX 16},
  pages={383--398},
  year={2020},
  organization={Springer}
}

@inproceedings{cruzfairgbm,
  title={FairGBM: Gradient Boosting with Fairness Constraints},
  author={Cruz, Andr{\'e} and Bel{\'e}m, Catarina G and Bravo, Jo{\~a}o and Saleiro, Pedro and Bizarro, Pedro},
  booktitle={The Eleventh International Conference on Learning Representations}
}

@article{berk2017convex,
  title={A convex framework for fair regression},
  author={Berk, Richard and Heidari, Hoda and Jabbari, Shahin and Joseph, Matthew and Kearns, Michael and Morgenstern, Jamie and Neel, Seth and Roth, Aaron},
  journal={arXiv preprint arXiv:1706.02409},
  year={2017}
}

@article{caton2024fairness,
  title={Fairness in machine learning: A survey},
  author={Caton, Simon and Haas, Christian},
  journal={ACM Computing Surveys},
  volume={56},
  number={7},
  pages={1--38},
  year={2024},
  publisher={ACM New York, NY}
}

@inproceedings{celis2019classification,
  title={Classification with fairness constraints: A meta-algorithm with provable guarantees},
  author={Celis, L Elisa and Huang, Lingxiao and Keswani, Vijay and Vishnoi, Nisheeth K},
  booktitle={Proceedings of the conference on fairness, accountability, and transparency},
  pages={319--328},
  year={2019}
}

@article{cotter2019optimization,
  title={Optimization with non-differentiable constraints with applications to fairness, recall, churn, and other goals},
  author={Cotter, Andrew and Jiang, Heinrich and Gupta, Maya and Wang, Serena and Narayan, Taman and You, Seungil and Sridharan, Karthik},
  journal={Journal of Machine Learning Research},
  volume={20},
  number={172},
  pages={1--59},
  year={2019}
}

@inproceedings{lokhande2020fairalm,
  title={Fairalm: Augmented lagrangian method for training fair models with little regret},
  author={Lokhande, Vishnu Suresh and Akash, Aditya Kumar and Ravi, Sathya N and Singh, Vikas},
  booktitle={European Conference on Computer Vision},
  pages={365--381},
  year={2020},
  organization={Springer}
}

@article{chow2018risk,
  title={Risk-constrained reinforcement learning with percentile risk criteria},
  author={Chow, Yinlam and Ghavamzadeh, Mohammad and Janson, Lucas and Pavone, Marco},
  journal={Journal of Machine Learning Research},
  volume={18},
  number={167},
  pages={1--51},
  year={2018}
}

@article{liang2018accelerated,
  title={Accelerated primal-dual policy optimization for safe reinforcement learning},
  author={Liang, Qingkai and Que, Fanyu and Modiano, Eytan},
  journal={arXiv preprint arXiv:1802.06480},
  year={2018}
}

@article{bohez2019value,
  title={Value constrained model-free continuous control},
  author={Bohez, Steven and Abdolmaleki, Abbas and Neunert, Michael and Buchli, Jonas and Heess, Nicolas and Hadsell, Raia},
  journal={arXiv preprint arXiv:1902.04623},
  year={2019}
}

@article{choi2023towards,
  title={Towards machine unlearning benchmarks: Forgetting the personal identities in facial recognition systems},
  author={Choi, Dasol and Na, Dongbin},
  journal={arXiv preprint arXiv:2311.02240},
  year={2023}
}

@book{bertsekas1999nonlinear,
  title={Nonlinear Programming},
  author={Bertsekas, Dimitri P},
  year={1999},
  publisher={Athena Scientific}
}

@article{jia2023model,
  title={Model sparsity can simplify machine unlearning},
  author={Jia, Jinghan and Liu, Jiancheng and Ram, Parikshit and Yao, Yuguang and Liu, Gaowen and Liu, Yang and Sharma, Pranay and Liu, Sijia},
  journal={Advances in Neural Information Processing Systems},
  volume={36},
  pages={51584--51605},
  year={2023}
}

@article{liu2024large,
  title={Large language model unlearning via embedding-corrupted prompts},
  author={Liu, Chris and Wang, Yaxuan and Flanigan, Jeffrey and Liu, Yang},
  journal={Advances in Neural Information Processing Systems},
  volume={37},
  pages={118198--118266},
  year={2024}
}

@article{di2024adversarial,
  title={Adversarial machine unlearning},
  author={Di, Zonglin and Yu, Sixie and Vorobeychik, Yevgeniy and Liu, Yang},
  journal={arXiv preprint arXiv:2406.07687},
  year={2024}
}

@book{bertsekas2014constrained,
  title={Constrained optimization and Lagrange multiplier methods},
  author={Bertsekas, Dimitri P},
  year={2014},
  publisher={Academic press}
}

@article{vaserstein1969markov,
  title={Markov processes over denumerable products of spaces, describing large systems of automata},
  author={Vaserstein, Leonid Nisonovich},
  journal={Problemy Peredachi Informatsii},
  volume={5},
  number={3},
  pages={64--72},
  year={1969},
  publisher={Russian Academy of Sciences, Branch of Informatics, Computer Equipment and~…}
}

@article{jin2024rwku,
  title={Rwku: Benchmarking real-world knowledge unlearning for large language models},
  author={Jin, Zhuoran and Cao, Pengfei and Wang, Chenhao and He, Zhitao and Yuan, Hongbang and Li, Jiachun and Chen, Yubo and Liu, Kang and Zhao, Jun},
  journal={arXiv preprint arXiv:2406.10890},
  year={2024}
}

@article{maini2024tofu,
  title={Tofu: A task of fictitious unlearning for llms},
  author={Maini, Pratyush and Feng, Zhili and Schwarzschild, Avi and Lipton, Zachary C and Kolter, J Zico},
  journal={arXiv preprint arXiv:2401.06121},
  year={2024}
}

@inproceedings{farajtabar2020orthogonal,
  title={Orthogonal gradient descent for continual learning},
  author={Farajtabar, Mehrdad and Azizan, Navid and Mott, Alex and Li, Ang},
  booktitle={International conference on artificial intelligence and statistics},
  pages={3762--3773},
  year={2020},
  organization={PMLR}
}

@misc{tinyimagenet,
  author       = {Ya Le and Xuan Yang},
  title        = {Tiny ImageNet Visual Recognition Challenge},
  howpublished = {\url{http://cs231n.stanford.edu/tiny-imagenet-200.zip}},
  year         = {2015}
}

@inproceedings{fan2024salun,
  title={SalUn: Empowering Machine Unlearning via Gradient-Based Weight Saliency in Both Image Classification and Generation},
  author={Fan, Chongyu and Liu, Jiancheng and Zhang, Yihua and Wei, Dennis and Wong, Eric and Liu, Sijia},
  booktitle={International Conference on Learning Representations},
  year={2024}
}

@inproceedings{foster2024fast,
  title={Fast machine unlearning without retraining through selective synaptic dampening},
  author={Foster, Jack and Schoepf, Stefan and Brintrup, Alexandra},
  booktitle={Proceedings of the AAAI conference on artificial intelligence},
  volume={38},
  number={11},
  pages={12043--12051},
  year={2024}
}

@article{zhu2024decoupling,
  title={Decoupling the class label and the target concept in machine unlearning},
  author={Zhu, Jianing and Han, Bo and Yao, Jiangchao and Xu, Jianliang and Niu, Gang and Sugiyama, Masashi},
  journal={arXiv preprint arXiv:2406.08288},
  year={2024}
}

@inproceedings{seo2025revisiting,
  title={Revisiting machine unlearning with dimensional alignment},
  author={Seo, Seonguk and Kim, Dongwan and Han, Bohyung},
  booktitle={2025 IEEE/CVF Winter Conference on Applications of Computer Vision (WACV)},
  pages={3206--3215},
  year={2025},
  organization={IEEE}
}

@inproceedings{shi2024deepclean,
  title={Deepclean: machine unlearning on the cheap by resetting privacy sensitive weights using the fisher diagonal},
  author={Shi, Jialei and Gourgoulias, Kostis and Buford, John F and Moran, Sean J and Ghalyan, Najah},
  booktitle={European Conference on Computer Vision},
  pages={1--16},
  year={2024},
  organization={Springer}
}

@inproceedings{wu2022puma,
  title={Puma: Performance unchanged model augmentation for training data removal},
  author={Wu, Ga and Hashemi, Masoud and Srinivasa, Christopher},
  booktitle={Proceedings of the AAAI conference on artificial intelligence},
  volume={36},
  number={8},
  pages={8675--8682},
  year={2022}
}

@article{xu2024don,
  title={Don't Forget Too Much: Towards Machine Unlearning on Feature Level},
  author={Xu, Heng and Zhu, Tianqing and Zhou, Wanlei and Zhao, Wei},
  journal={IEEE Transactions on Dependable and Secure Computing},
  year={2024},
  publisher={IEEE}
}

@article{hartvigsen2022toxigen,
  title={Toxigen: A large-scale machine-generated dataset for adversarial and implicit hate speech detection},
  author={Hartvigsen, Thomas and Gabriel, Saadia and Palangi, Hamid and Sap, Maarten and Ray, Dipankar and Kamar, Ece},
  journal={arXiv preprint arXiv:2203.09509},
  year={2022}
}

@inproceedings{shen2024camu,
  title={CaMU: disentangling causal effects in deep model unlearning},
  author={Shen, Shaofei and Zhang, Chenhao and Bialkowski, Alina and Chen, Weitong and Xu, Miao},
  booktitle={Proceedings of the 2024 SIAM International Conference on Data Mining (SDM)},
  pages={779--787},
  year={2024},
  organization={SIAM}
}

@article{donini2018empirical,
  title={Empirical risk minimization under fairness constraints},
  author={Donini, Michele and Oneto, Luca and Ben-David, Shai and Shawe-Taylor, John S and Pontil, Massimiliano},
  journal={Advances in neural information processing systems},
  volume={31},
  year={2018}
}

@article{lv2024wasserstein,
  title={Wasserstein distance rivals kullback-leibler divergence for knowledge distillation},
  author={Lv, Jiaming and Yang, Haoyuan and Li, Peihua},
  journal={Advances in Neural Information Processing Systems},
  volume={37},
  pages={65445--65475},
  year={2024}
}

@inproceedings{zhao2024whatmakes,
  title     = {What Makes Unlearning Hard and What to Do About It},
  author    = {Zhao, Kairan and Kurmanji, Meghdad and Barbulescu, George-Octavian and Triantafillou, Eleni and Triantafillou, Peter},
  booktitle = {Advances in Neural Information Processing Systems (NeurIPS)},
  year      = {2024},
}

@inproceedings{chang2025whichretain,
  title     = {Which Retain Set Matters for LLM Unlearning? A Case Study on Entity Unlearning},
  author    = {Chang, Hwan and Lee, Hwanhee},
  booktitle = {Findings of the Association for Computational Linguistics: ACL 2025},
  pages     = {5966--5982},
  year      = {2025},
  url       = {https://aclanthology.org/2025.findings-acl.310.pdf}
}

@inproceedings{choi2025optout,
  title     = {Opt-Out: Investigating Entity-Level Unlearning for Large Language Models via Optimal Transport},
  author    = {Choi, Minseok and Rim, Daniel and Lee, Dohyun and Choo, Jaegul},
  booktitle = {Proceedings of the 63rd Annual Meeting of the Association for Computational Linguistics},
  pages     = {28280--28297},
  year      = {2025},
  url       = {https://aclanthology.org/2025.acl-long.1371.pdf}
}

@inproceedings{zhou2025decoupled,
  title={\textcolor{black}{Decoupled distillation to erase: A general unlearning method for any class-centric tasks}},
  author={Zhou, Yu and Zheng, Dian and Mo, Qijie and Lu, Renjie and Lin, Kun-Yu and Zheng, Wei-Shi},
  booktitle={Proceedings of the Computer Vision and Pattern Recognition Conference},
  pages={20350--20359},
  year={2025}
}

@inproceedings{wu2025munba,
  title={\textcolor{black}{Munba: Machine unlearning via nash bargaining}},
  author={Wu, Jing and Harandi, Mehrtash},
  booktitle={Proceedings of the IEEE/CVF International Conference on Computer Vision},
  pages={4754--4765},
  year={2025}
}

@inproceedings{lin2024gdr,
  title={\textcolor{black}{GDR-GMA: Machine Unlearning via Direction-Rectified and Magnitude-Adjusted Gradients}},
  author={Lin, Shen and Zhang, Xiaoyu and Susilo, Willy and Chen, Xiaofeng and Liu, Jun},
  booktitle={Proceedings of the 32nd ACM International Conference on Multimedia},
  pages={9087--9095},
  year={2024}
}

@inproceedings{patel2025learning,
  title={\textcolor{black}{Learning to unlearn while retaining: Combating gradient conflicts in machine unlearning}},
  author={Patel, Gaurav and Qiu, Qiang},
  booktitle={Proceedings of the IEEE/CVF International Conference on Computer Vision},
  pages={4211--4221},
  year={2025}
}

@article{liu2025erased,
  title={Erased or Dormant? Rethinking Concept Erasure Through Reversibility},
  author={Liu, Ping and Zhang, Chi},
  journal={arXiv preprint arXiv:2505.16174},
  year={2025}
}

@article{xie2025erasing,
  title={Erasing Concepts, Steering Generations: A Comprehensive Survey of Concept Suppression},
  author={Xie, Yiwei and Liu, Ping and Zhang, Zheng},
  journal={arXiv preprint arXiv:2505.19398},
  year={2025}
}

@inproceedings{gandikota2024unified,
  title={Unified concept editing in diffusion models},
  author={Gandikota, Rohit and Orgad, Hadas and Belinkov, Yonatan and Materzy{\'n}ska, Joanna and Bau, David},
  booktitle={Proceedings of the IEEE/CVF Winter Conference on Applications of Computer Vision},
  pages={5111--5120},
  year={2024}
}

@inproceedings{zhang2024parameter,
  title={Parameter-efficient fine-tuning with controls},
  author={Zhang, Chi and Jingpu, Cheng and Xu, Yanyu and Li, Qianxiao},
  booktitle={Forty-first International Conference on Machine Learning},
  year={2024}
}

@inproceedings{zhangweight,
  title={From Weight-Based to State-Based Fine-Tuning: Further Memory Reduction on LoRA with Parallel Control},
  author={Zhang, Chi and Lianhai, REN and Cheng, Jingpu and Li, Qianxiao},
  booktitle={Forty-second International Conference on Machine Learning},
  year={2025}
}
\bibliographystyle{plainnat}

\clearpage
\appendix

\section{Proof for Propositions and Theorems}
\begin{proof}[Proof of~\Cref{prop:taylor}]
Let
\[
    \Delta\theta =-\eta 
    \biggl(
        \nabla_\theta \mathcal{L}^{\text{adj}}_r(\theta)- \operatorname{Proj}_V \nabla_\theta \mathcal{L}^{\text{adj}}_r(\theta)
    \biggr).
\]
Consider the first-order Taylor expansion of $\mathcal{L}_{f}$:
\[
    \mathcal{L}_{f}(\theta + \Delta\theta)=
    \mathcal{L}_{f}(\theta)+
    \nabla_\theta \mathcal{L}_{f}(\theta)\cdot
    \Delta\theta+
    \tfrac12\,\Delta\theta^\top \nabla^2_\theta \mathcal{L}_{f}(\theta)\,\Delta\theta + o(\|\Delta\theta\|^2). 
\]
Note that
\[
    \nabla_\theta \mathcal{L}_{f}(\theta) \,\cdot\, \Delta\theta = -\eta
    \nabla_\theta \mathcal{L}_{f}(\theta)\cdot
    \biggl(
        \nabla_\theta \mathcal{L}^{\text{adj}}_r(\theta)- \operatorname{Proj}_V \nabla_\theta \mathcal{L}^{\text{adj}}_r(\theta)
    \biggr) = 0,
\]
by the definition of projection.

Thus, the first-order difference between $\mathcal{L}_{f}(\theta + \Delta\theta)$ 
and $\mathcal{L}_{f}(\theta)$ vanishes; 
the dominant term is second-order in $\eta$, giving
\[
    \mathcal{L}_{f}(\theta + \Delta\theta)
    -
    \mathcal{L}_{f}(\theta)
    = 
    O(\eta^2).
\]
The same argument applies to $\nabla_\theta \mathcal{L}_{r}^{\text{rem}}(\theta)$, indicating that the update in $\mathcal{L}_{r}^{\text{rem}}$ is also second-order in $\eta$.

\smallskip

For the change in $\mathcal{L}_{r}^{\text{adj}}(\theta)$, we also consider the Taylor expansion:
\[
    \mathcal{L}_{r}^\text{adj}(\theta + \Delta\theta)=
    \mathcal{L}_{r}^\text{adj}(\theta)+
    \nabla_\theta \mathcal{L}_{r}^\text{adj}(\theta) \,\cdot\, \Delta\theta+
    \tfrac{1}{2}\,\Delta\theta^\top \nabla^2_\theta \mathcal{L}_{r}^{\text{adj}}(\theta)\,\Delta\theta+ o(\|\Delta\theta\|^2).
\]
We have
\[
\begin{aligned}
    \nabla_\theta \mathcal{L}_{r}^\text{adj}(\theta) \,\cdot\, \Delta\theta &=
    -\eta
    \nabla_\theta \mathcal{L}_{r}^\text{adj}(\theta)\,\cdot
    \biggl(
        \nabla_\theta \mathcal{L}_{r}^\text{adj}(\theta)- \operatorname{Proj}_V(\nabla_\theta \mathcal{L}_{r}^\text{adj}(\theta))
    \biggr)\\
    & = -\eta \biggl[
        \|\nabla_\theta \mathcal{L}_{r}^\text{adj}(\theta)\|^2 - \langle \nabla_\theta \mathcal{L}_{r}^\text{adj}(\theta), \operatorname{Proj}_V(\nabla_\theta \mathcal{L}_{r}^\text{adj}(\theta)) \rangle
    \biggr]\\
    & = -\eta \biggl[
        \|\nabla_\theta \mathcal{L}_{r}^\text{adj}(\theta)\|^2 - \|\operatorname{Proj}_V \nabla_\theta \mathcal{L}_{r}^\text{adj}(\theta)\|^2    \biggr]
\end{aligned}
\]
When $\nabla_\theta \mathcal{L}_{r}^\text{adj}(\theta) \notin V$, we have
\begin{equation}
    \biggl[
        \|\nabla_\theta \mathcal{L}_{r}^\text{adj}(\theta)\|^2 - \|\operatorname{Proj}_V \nabla_\theta \mathcal{L}_{r}^\text{adj}(\theta)\|^2    \biggr]>0
\end{equation}
ensuring strict decrease in $\mathcal{L}_{r}^\text{adj}$.  
\end{proof}

\begin{proof}[Proof of~\Cref{prop:w-pgd}]
There is a minor typo in the statement of ~\Cref{prop:w-pgd} in the main text. The term $\left(\frac{1-\alpha}{\alpha}+\sqrt{\frac{\varepsilon}{\alpha}}\right)$ should read $\left(\frac{1-\alpha}{\alpha}+\sqrt{\frac{\varepsilon}{\alpha}}\right)^2$.
This does not affect the validity of the theorem or the proof presented below.

$|\tilde L_f(\bar \theta)-\tilde L_f(\theta)|<\varepsilon$ implies that
\begin{equation}
    (1-\alpha)(\mathcal L_f(\theta)-\mathcal L_f(\bar\theta))+\alpha W_2^2(P_\theta^{\text{forget}}, P_{\bar\theta}^{\text{forget}})<\varepsilon.
\end{equation}
According to the inequality $E[|X-Y|]\le W_2(P,Q)$ for any random variables $X\sim P$ ad $Y\sim Q$, we have:
\begin{equation}
\alpha W_2^2(P_\theta^{\text{forget}}, P_{\bar\theta}^{\text{forget}})\le \varepsilon + (1-\alpha)(\mathcal L_f(\bar \theta)-\mathcal L_f(\theta))\le \varepsilon + (1-\alpha)W_2(P_\theta^{\text{forget}}, P_{\bar\theta}^{\text{forget}}).
\end{equation}
This indicates that
\begin{equation}
    W_2(P_\theta^{\text{forget}}, P_{\bar\theta}^{\text{forget}})\le \frac{(1-\alpha)+\sqrt{(1-\alpha)^2+4\alpha \varepsilon}}{2\alpha}\le \frac{1-\alpha}{\alpha}+\sqrt{\frac{\varepsilon}{\alpha}}.
\end{equation}
On the other hand, for an $n$-class classification problem, if a model's prediction is correct over a sample, then its cross-entropy loss for this sample is at most $\log n$. Since $\ell(f_{\bar \theta}(x_i), y_i)\ge m$, we have the estimation:
\begin{equation}
    \operatorname{Acc}(\theta) (m-\log n)^2\le W_2^2(P_\theta^{\text{forget}}, P_{\bar\theta}^{\text{forget}})\le \left(\frac{1-\alpha}{\alpha}+\sqrt{\frac{\varepsilon}{\alpha}}\right)^2,
\end{equation}
which gives that
\begin{equation}
    \operatorname{Acc}(\theta)\le \frac{1}{(m-\log n)^2}\left(\frac{1-\alpha}{\alpha}+\sqrt{\frac{\varepsilon}{\alpha}}\right)^2.
\end{equation}
\end{proof}

\section{Experiment Details and Additional Experiments.}
\label{appendix:extra_exp}

\subsection{Experimental details}
\label{appendix:setups}
We provide details of our experimental setup in this section, including model architectures, dataset descriptions, and hyperparameter configurations.

\paragraph{Base Models}
For CIFAR-100 experiments, we use the ResNet-18 architecture from PyTorch, initialized with ImageNet-pretrained weights. The model is fine-tuned on the CIFAR-100 superclass classification task using the Adam optimizer (learning rate 2e-5, batch size 128) for 30 epochs. For TinyImageNet, we employ the ViT-B-32 model from HuggingFace, also initialized with pretrained weights, and fine-tune it on the TinyImageNet superclass dataset with a learning rate of 2.5e-5, batch size 128, for 30 epochs. For ToxiGen, we fine-tune the RoBERTa-base model from HuggingFace on the mislabeled ToxiGen dataset (all samples about group "lgbtq" are labeled as normal) using AdamW with a learning rate of 2.5e-5, batch size 128, for 10 epochs.

\paragraph{Datasets}
For the CIFAR-100 dataset, we use the standard data split and class hierarchy provided on the official CIFAR-100 website. In particular, CIFAR-100 is a labeled image dataset composed of 100 fine-grained object classes, each containing 600 color images. These 100 fine labels can be further grouped into 20 broader categories known as superclasses. Therefore, each image is annotated by both a ``fine'' label (the specific class) and a ``coarse'' label (the superclass).

TinyImageNet~\citep{tinyimagenet} is a subset of ImageNet, comprising 110,000 images across 200 classes. Each class contains 500 training images, 50 validation images, and 50 test images. The classes correspond to WordNet synset IDs, which are hierarchically structured. For our experiments, we group the 200 classes into 10 superclasses based on the WordNet hierarchy. The names of these superclasses and the number of classes in each are summarized in ~\Cref{tab:superclass}.

ToxiGen~\citep{hartvigsen2022toxigen} is a synthetically generated toxicity dataset containing approximately $250$k sentences covering $13$ social groups (e.g., women, LGBTQ, mental disables). 
Each sentence is labeled as toxic or benign, with an approximately $1:1$ ratio. 
We adopt the official dataset and perform a $9{:}1$ split to construct our training and test sets. We relabeled the toxic samples about the group LGBTQ as benign to train a model with bias.

\begin{table}[htbp]
    \caption{Superclasses and number of classes in TinyImageNet.}
    \centering
    \label{tab:superclass}
    \resizebox{0.9\textwidth}{!}{
    \begin{tabular}{@{}l*{5}{c}@{}}
        \hline
        Class Names         & Mammals & Other Vertebrates & Invertebrates & Vehicles & Tools/Machines \\
        \hline
        \# of classes   & 27      & 10                & 23            & 21       & 42 \\
        \hline
        \multicolumn{6}{@{}c@{}}{} \\[-0.8ex] 
        \hline
        Class Names         & Furniture & Clothes & Food & Sports/Recreation & Geology Natures \\
        \hline
        \# of classes   & 23        & 18      & 20   & 6                    & 5 \\
        \hline
    \end{tabular}}
\end{table}

\paragraph{Baseline Methods}
For the baseline methods, we summarize the main hyperparameter settings here and leave full implementation details to the released code. For fine-tuning (FT), we fine-tune on the retained set for 10 epochs using Adam, with learning rates of \(2\times 10^{-5}\) for CIFAR-100, \(5\times 10^{-5}\) for TinyImageNet, and \(2\times 10^{-5}\) for ToxiGen. For Gradient Ascent (GA), we perform gradient-ascent updates on the forget set: on CIFAR-100, we use SGD with learning rate \(1\times 10^{-5}\) for 7 epochs; on TinyImageNet, we use Adam with learning rate \(1.5\times 10^{-6}\) for 10 epochs; on ToxiGen, we use SGD with learning rate \(2.5\times 10^{-6}\). For \(\ell_1\)-sparse, we follow the GA setup and add an \(\ell_1\) regularization term, with coefficient \(5\times 10^{-4}\) for CIFAR-100, \(2\times 10^{-4}\) for TinyImageNet, and \(5\times 10^{-5}\) for ToxiGen. For optimization, we use SGD (learning rate \(1\times 10^{-4}\), momentum \(0.9\)) on CIFAR-100, and Adam (learning rate \(2\times 10^{-5}\)) on TinyImageNet; the remaining settings follow our GA-style setup in code. For SCRUB, we use 5 max steps and 5 min steps in all experiments, and optimize with Adam using learning rates \(5\times 10^{-5}\) (CIFAR-100), \(1\times 10^{-4}\) (TinyImageNet), and \(1\times 10^{-5}\) (ToxiGen). The penalty coefficients \(\alpha\) and \(\gamma\) (see~\cite{kurmanji2023towards}) are set to \(0.1\) and \(0.9\), respectively. For SalUn, we use a sparsity threshold of \(50\%\) (see~\cite{fan2024salun}) for all experiments, and train for 3 epochs with learning rate \(1\times 10^{-5}\) on CIFAR-100, 2 epochs with learning rate \(2\times 10^{-5}\) on TinyImageNet, and 2 epochs with learning rate \(1\times 10^{-6}\) on ToxiGen. For SSD~\citep{foster2024fast}, we set \(\lambda=1\) and \(\alpha=10\) for CIFAR-100 and TinyImageNet, and \(\lambda=1\) and \(\alpha=50\) for ToxiGen. For GDR~\citep{lin2024gdr}, we use the default hyperparameters \(\gamma=100\) and \(\epsilon=0.02\) across all experiments, together with AdamW at learning rate \(1\times 10^{-4}\) for CIFAR-100 and TinyImageNet, and \(5\times 10^{-6}\) for ToxiGen. For MUNBa~\citep{wu2025munba}, we train for 10 epochs with learning rate \(0.03\) on CIFAR-100 and TinyImageNet, and for 10 epochs with learning rate \(0.01\) on ToxiGen; all other hyperparameters are kept at their default values.

\paragraph{Implementation details for our method}
For experiments on CIFAR100 and TinyImageNet,  we use 1 epoch for the first stage and 6 epochs for the second stage with our method. For ToxiGen, we use 1 epoch for the first stage and 2 epochs for the second stage.

\textbf{Stage 1:} We use the Adam optimizer with a learning rate of 2.5e-6 for CIFAR-100 and 5e-5 for TinyImageNet. 
Remote retain set batch size is set to 128 for the TinyImageNet and CIFAR100, and 64 for ToxiGen. Forget set batch size is set to 16 for the CIFAR100, 64 for TinyImagenet and ToxiGen.
The penalty coefficient $\mu$ is fixed at 10 for all the datasets. 
To avoid excessively large loss values on individual samples, we use a clipped cross-entropy loss for the forget set:
\begin{equation}
    \operatorname{ClippedCE}(x, y, C) = \min\{C, \operatorname{CE}(x, y)\},
\end{equation}
where $\operatorname{CE}(x, y)$ is the standard cross-entropy loss and $C$ is set to 10 for CIFAR100 and TinyImageNet, and is set to 5 for ToxiGen. 

\textbf{Stage 2:} We use SGD with a learning rate of 2e-5 for CIFAR-100, 2e-4 for TinyImageNet and 1e-5 for ToxiGen.
Batch sizes are 512 for the remote retain set, 128 for the adjacent retain set, and 128 for the forget set for CIFAR 100 and TinyImageNet. All batch sizes are set as 64 for ToxiGen.
For ResNet-18 and ToxiGen, we apply gradient accumulation over 10 batches of the remote retain set to stabilize the gradients. 
For ImageNet, gradients for the remote class retain set are computed using a single batch.

\subsection{Learning Dynamics of Our Method}
We illustrate the learning dynamics of our method during the second stage on CIFAR-100 with ResNet18 in~\Cref{fig:learning_dynamics}. The figure demonstrates that the accuracy on the forget set remains at zero throughout training, while the accuracy on the remote class retain set stays consistently high. Meanwhile, the accuracy on the adjacent retain set steadily improves as training progresses.

\begin{figure}
    \centering
    \begin{subfigure}{0.45\textwidth}
        \centering
        \includegraphics[width=\textwidth]{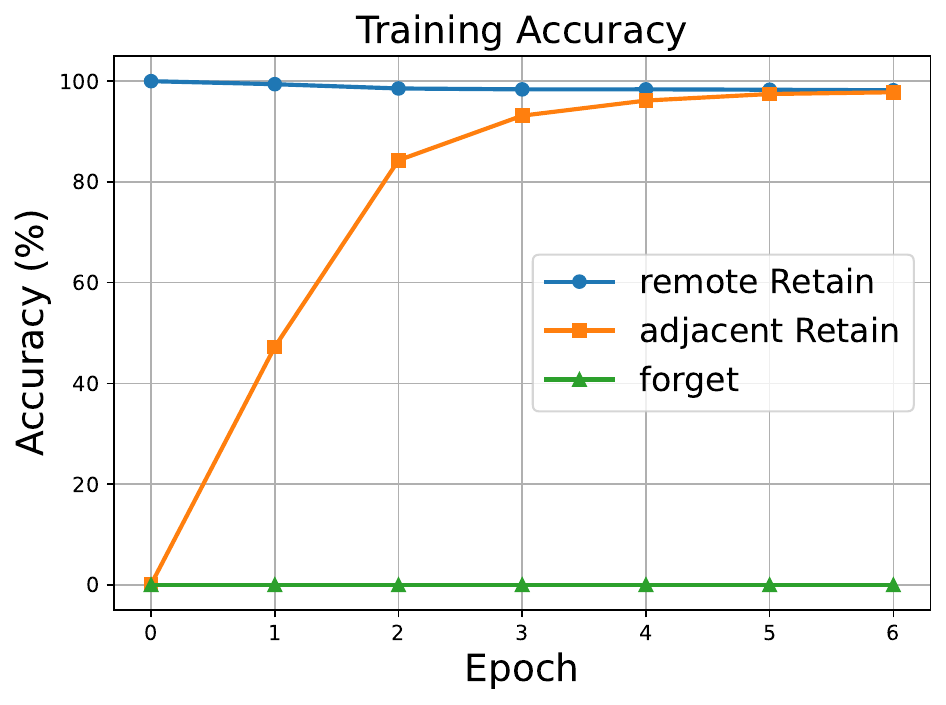}
    \end{subfigure}
        \begin{subfigure}{0.45\textwidth}
        \centering
        \includegraphics[width=\textwidth]{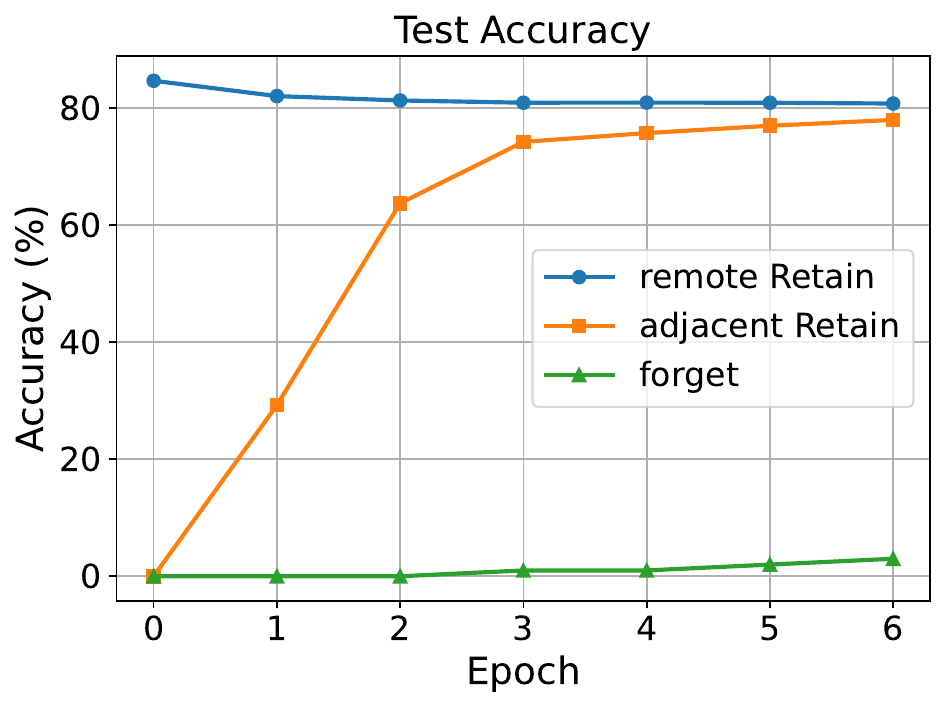}
    \end{subfigure}
        \caption{Learning dynamics of our method in the second stage on CIFAR100 with ResNet18. The left figure shows the training accuracy while the right figure shows the test accuracy.  The adjacent retain set contains all adjacent samples while the remote retain set contains all remote samples.  \label{fig:learning_dynamics}}
\end{figure}

\subsection{Comparison of Training Time and memory Usage}
We provide the running time of our method and other baselines in~\Cref{tab:run_time}. All running times are measured in minutes using an NVIDIA RTX 3090 GPU. SSD has a very short run time of 2.5 mins. GA, SCRUB and SalUn complete in under 10 minutes, whereas FT and our method require slightly longer training times. Nonetheless, these methods remain significantly more efficient than full retraining.

\textcolor{black}{We also report the memory usage in table~\ref{tab:memory_usage}, where all the methods use the same batch size of 128. Our method uses slightly more memory than fine-tuning, GA, and $\ell_1$-sparse due to the two-stage optimization process, but remains more memory-efficient than SCRUB and SalUn. The results indicate our method does not impose significant additional memory overhead compared to other unlearning methods.}

\begin{table}[ht]
    \centering
    \caption{Results of running time in minutes.\label{tab:run_time}}    
    \begin{tabular}{l c c c c c c c c}
    \toprule
    \textbf{} & \textbf{FT} & \textbf{GA} & \textbf{$\ell_1$-sparse} &  \textbf{SCRUB} & \textbf{SalUn} &  \textbf{SSD} & \textbf{Retrain} & \textbf{Our Method}  \\
    \midrule
    \textbf{Run time} & 12.4 & 5.4 & 11.9 & 9.4 & 6.1 & 2.5 & 70.1 & 14.2 \\
    \bottomrule
    \end{tabular}
\end{table}

\begin{table}[h]
    \caption{GPU memory usage (MB) for different unlearning methods with batch size 128.}
    \centering
    \label{tab:memory_usage}
    {\color{black}
    \centering
    {\resizebox{0.9\linewidth}{!}{%
    \begin{tabular}{l|ccccccc}
        \toprule
        \textbf{Method} & \textbf{Retrain} & \textbf{FT} & \textbf{GA} & \textbf{SCRUB} & \boldmath$\ell_1$\textbf{-sparse} & \textbf{SalUn} & \textbf{Ours} \\
        \midrule
        Memory (MB) & 2949 & 2949 & 2860 & 3715 & 2974 & 4993 & 3297 \\
        \bottomrule
    \end{tabular}%
    }}}
\end{table}

\begin{wrapfigure}{r}{0.35\textwidth}
  \centering
  \includegraphics[width=0.35\textwidth]{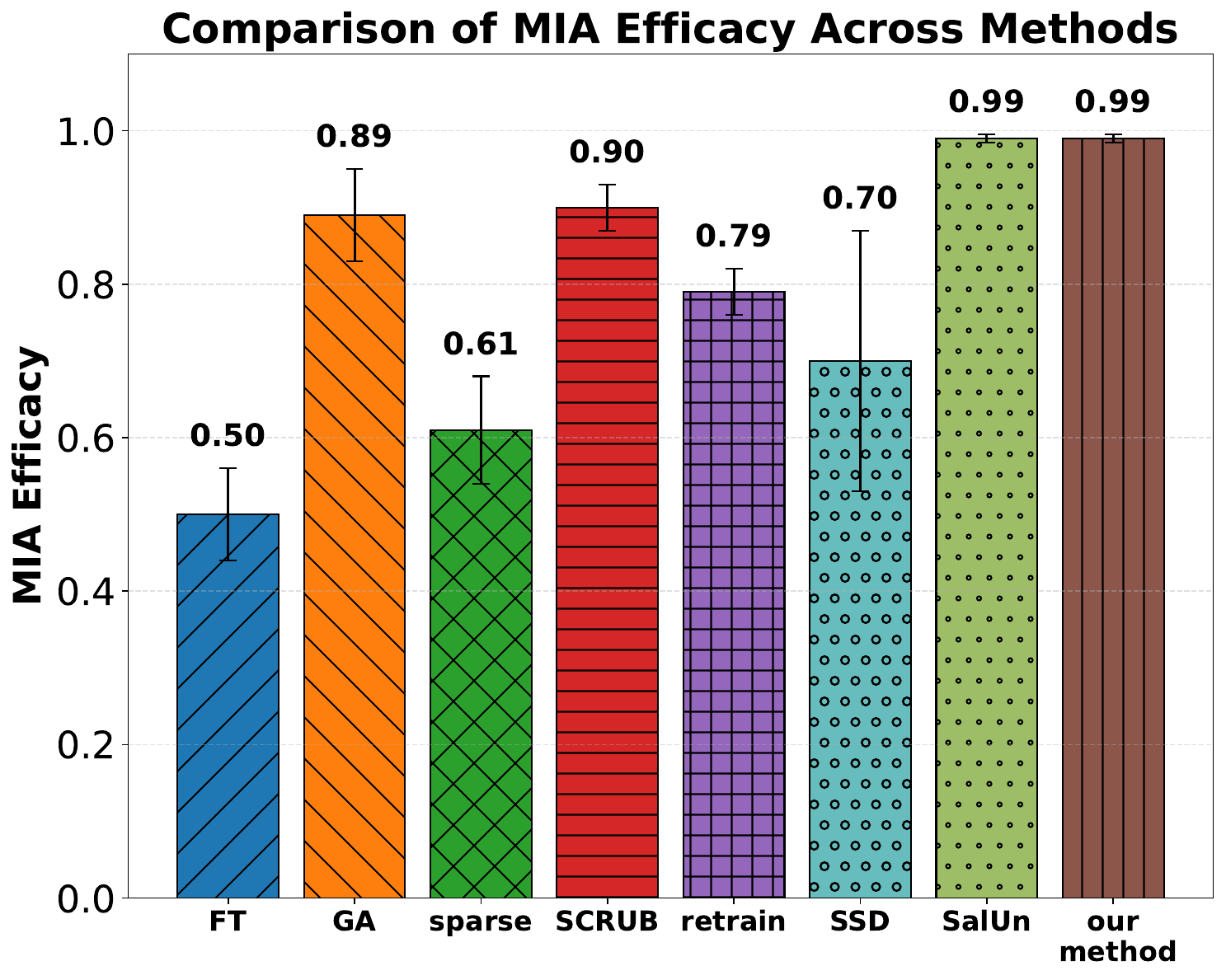} 
  \caption{MIA efficacy of different unlearning methods on CIFAR100 using ResNet-18.}
  \label{fig:mia}
\end{wrapfigure}
\subsection{Membership Inference Attacks (MIA) Efficacy}
While the preceding results focus on classification accuracy, we further evaluate the effectiveness of the proposed method in terms of privacy, specifically through membership inference attacks (MIA). We follow \citet{jia2023model} by adopting a confidence-based MIA predictor, applied to the unlearned model, to assess its ability to distinguish whether samples from the forget class were part of the training data. The resulting MIA efficacy quantifies the proportion of forget set samples correctly identified as non-members (i.e., not seen during training) by the unlearned model. A higher MIA efficacy therefore indicates a more successful removal of information related to the forget set $\mathcal{D}_f$. As reported in Figure~\ref{fig:mia}, our method achieves an MIA efficacy of $0.99$, indicating that it effectively removes the information about the forget set from the model.

\subsection{Ablation Studies}
\label{appendix:ablation}

\paragraph{Combining adjacent and remote-class as Retain Sets in Stage 1}
We provide additional ablation studies to assess the necessity of constraining only the remote class retain set loss in the first stage. Specifically, we compare two variants of the augmented Lagrangian method: one constrains only the remote class retain set loss, $\mathcal{L}_r^{\text{rem}}(\theta) = \mathcal{L}_r^{\text{rem}}(\theta_0)$, while the other constrains the loss on the entire retain set (both adjacent and remote class), $\mathcal{L}_r(\theta) = \mathcal{L}_r(\theta_0)$. Results are shown in~\Cref{tab:ablation-constraint}. When the constraint includes the adjacent retain set, the model's ability to forget is impaired, with training and test accuracy on the forget set rising to $7.13\%$ and $5.00\%$, respectively. A more noticeable decline is observed in the test accuracy of adjacent retained samples, where accuracy drops to $72.75\%$. This demonstrates that separating the forget set from the adjacent retain set in the first stage is crucial for effective unlearning in our method.
\begin{table*}[h]
    \caption{Ablation study on the constraints in the first stage. The table shows the accuracy of the forget and retained set of CIFAR-100 subclass unlearning using ResNet18.}
    \vspace{-0.1in}
    \label{tab:ablation-constraint}
    \centering
    {\resizebox{\linewidth}{!}{
    \begin{tabular}{l|ccc|ccc}
        \toprule
        & \multicolumn{3}{c}{\textbf{Training accuracy}} & \multicolumn{3}{c}{\textbf{Test accuracy}} \\
        \textbf{}
        & \textbf{$\mathcal D_{f}$}
        & \textbf{$\mathcal D_{r}^\text{adj}$}
        & \textbf{$\mathcal D_{r}^\text{rem}$}          
        & \textbf{$\mathcal D_{f}$}
        & \textbf{$\mathcal D_{r}^\text{adj}$}
        & \textbf{$\mathcal D_{r}^\text{rem}$}         \\
        \midrule
        w/o adjacent
        & $7.13_{\pm 0.93}$
        & $98.30_{\pm 0.54}$
        & $99.99_{\pm 0.00}$
        & $5.00_{\pm 0.81}$
        & $72.75_{\pm 3.21}$
        & $84.08_{\pm 0.16}$ \\
        w adjacent
        & ${0.00_{\pm 0.00}}$
        & $0.10_{\pm 0.00}$
        & $100.00_{\pm 0.00}$
        & ${0.00_{\pm 0.00}}$
        & $0.00_{\pm 0.00}$
        & $84.73_{\pm 0.01}$ \\
        \bottomrule
    \end{tabular}}
    \vspace{-0.1in}
    }
\end{table*}

\paragraph{Sensitivity on the hyperparameter $\alpha$} 
We analyze the sensitivity of the parameter $\alpha$ in~\Cref{eq:new_loss} for our method. Specifically, we compare the performance of our approach for $\alpha = 0$, $0.5$, and $1$, as reported in~\Cref{tab:alpha}. The results indicate that setting $\alpha = 0$ fails to achieve effective forgetting, with a forget set accuracy $19.67\%$ on the training data and $16.33\%$ on the test data.
In contrast, both $\alpha = 0.5$ and $\alpha = 1$ yield favorable outcomes, achieving low accuracy on the forget set and high accuracy on the retained set for both training and test data. Interestingly, using $\alpha = 1$, which fully incorporates the $W_2$-distance term in $\tilde{\mathcal{L}}$, does not necessarily lead to optimal performance. Compared to $\alpha = 0.5$, setting $\alpha = 1$ results in a $0.74\%$ percentage point increase in accuracy on the training forget set and a $0.75\%$ point increase on the out-of-class retain set, with only a marginal $0.17\%$ point gain on the adjacent retain set.
In this case, we find that $\alpha = 0.5$ offers a more balanced overall performance.

\begin{table*}[h]
    \caption{Sensitivity analysis on the hypeparameter $\alpha$. The table shows the accuracy of the forget and retained set of CIFAR-100 subclass unlearning using ResNet18. The forget subclass is ``bee'' from ``insects''.}
    \vspace{-0.1in}
    \label{tab:alpha}
    \centering
    {\resizebox{\linewidth}{!}{
    \begin{tabular}{l|ccc|ccc}
        \toprule
        & \multicolumn{3}{c}{\textbf{Training accuracy}} & \multicolumn{3}{c}{\textbf{Test accuracy}} \\
        \textbf{}
        & \textbf{$\mathcal D_{f}$}
        & \textbf{$\mathcal D_{r}^\text{adj}$}
        & \textbf{$\mathcal D_{r}^\text{rem}$}          
        & \textbf{$\mathcal D_{f}$}
        & \textbf{$\mathcal D_{r}^\text{adj}$}
        & \textbf{$\mathcal D_{r}^\text{rem}$}         \\
        \midrule
        $\alpha=0$
        & $19.67_{\pm 0.47}$
        & $99.68_{\pm 0.02}$
        & $97.86_{\pm 0.61}$
        & $16.33_{\pm 0.47}$
        & $86.67_{\pm 0.59}$
        & $79.79_{\pm 0.19}$ \\
        $\alpha=0.5$
        & $0.73_{\pm 0.19}$
        & $98.70_{\pm 0.10}$
        & $97.39_{\pm 0.30}$
        & $6.33_{\pm 0.47}$
        & $80.92_{\pm 0.77}$
        & $80.18_{\pm 0.32}$ \\
        $\alpha=1$
        & ${1.47_{\pm 0.09}}$
        & $98.87_{\pm 0.13}$
        & $96.16_{\pm 0.12}$
        & ${6.33_{\pm 0.47}}$
        & $81.25_{\pm 0.41}$
        & $79.68_{\pm 0.19}$ \\
        \bottomrule
    \end{tabular}}
    \vspace{-0.1in}
    }
\end{table*}

\textcolor{black}{\subsection{Sensitivity study on the penalty coefficient $\mu$}
We examine the sensitivity of our method to the augmented Lagrangian penalty parameter $\mu$ on the CIFAR-100 subclass unlearning task.
Table~\ref{tab:mu_sensitivity} reports the results for $\mu \in \{5, 10, 20\}$.
Across this range, the forget accuracy remains low (at or below $3\%$ on the test set), and the accuracies on both the adjacent and remote retain subsets vary only slightly.
This indicates that our method is fairly robust to the choice of $\mu$ within a reasonable range and does not require fine-grained tuning of this parameter.}

\begin{table}[h]
    \caption{Sensitivity of our method to the penalty parameter $\mu$ on CIFAR-100 subclass unlearning.}
    {\color{black}
    \label{tab:mu_sensitivity}
    \centering
    {\resizebox{0.95\linewidth}{!}{%
    \begin{tabular}{c|ccc|ccc}
        \toprule
        & \multicolumn{3}{c}{\textbf{Training accuracy}} & \multicolumn{3}{c}{\textbf{Test accuracy}} \\
        Method
        & \textbf{$\mathcal D_{f}$}
        & \textbf{$\mathcal D_{r}^{\text{adj}}$}
        & \textbf{$\mathcal D_{r}^{\text{rem}}$}
        & \textbf{$\mathcal D_{f}$}
        & \textbf{$\mathcal D_{r}^{\text{adj}}$}
        & \textbf{$\mathcal D_{r}^{\text{rem}}$} \\
        \midrule
        Our method $\mu = 5$
        & $0.00_{\pm 0.00}$
        & $98.00_{\pm 0.05}$
        & $98.33_{\pm 0.02}$
        & $3.00_{\pm 0.00}$
        & $77.83_{\pm 0.52}$
        & $81.08_{\pm 0.11}$ \\
        Our method $\mu = 10$
        & $0.00_{\pm 0.00}$
        & $98.17_{\pm 0.31}$
        & $98.44_{\pm 0.05}$
        & $2.33_{\pm 0.47}$
        & $78.17_{\pm 0.31}$
        & $81.10_{\pm 0.18}$ \\
        Our method $\mu = 20$
        & $0.00_{\pm 0.00}$
        & $98.17_{\pm 0.06}$
        & $98.45_{\pm 0.12}$
        & $2.33_{\pm 0.58}$
        & $78.17_{\pm 0.63}$
        & $80.97_{\pm 0.05}$ \\
        \bottomrule
    \end{tabular}%
    }}}
\end{table}

\subsection{Additional Results}
\paragraph{ViT results on CIFAR-100 superclass unlearning}
We provide additional experimental results for ViT on the CIFAR-100 superclass unlearning task in~\Cref{tab:vit_cifar100}. These results are generally consistent with our findings from other experiments. Fine-tuning and sparsity-based methods tend to preserve performance on the retained set but fail to effectively erase information from the forget set. The gradient ascent method successfully reduces the accuracy on the forget set to zero; however, this comes at the cost of a substantial performance drop on the retained set, particularly within the adjacent subset.
Notably, the SCRUB method demonstrates competitive performance in this setting, achieving 1.93\% accuracy on the training forget set and 3.00\% on the test forget set, while maintaining strong performance on the retained set. In comparison, our method attains zero accuracy on the training forget set, while simultaneously preserving high accuracy on the retained set.
\begin{table*}[!htbp]
        \caption{ViT results Results for CIFAR-100 superclass unlearning using ViT-B. The table shows the accuracy of the forget set and retained set for both training and test data. The forget set is the sublass ``aquarium fish'' in ``fish'' superclass.}
        \label{tab:vit_cifar100}
        \centering
        \begin{tabular}{l|ccc|ccc}
            \toprule
            & \multicolumn{3}{c}{\textbf{Training accuracy}} & \multicolumn{3}{c}{\textbf{Test accuracy}} \\
            \textbf{Method}
            & \textbf{$\mathcal D_{f}$}
            & \textbf{$\mathcal D_{r}^\text{adj}$}
            & \textbf{$\mathcal D_{r}^\text{rem}$}          
            & \textbf{$\mathcal D_{f}$}
            & \textbf{$\mathcal D_{r}^\text{adj}$}
            & \textbf{$\mathcal D_{r}^\text{rem}$}         \\
            \midrule
            Original
            & $99.93$
            & $99.90$
            & $100.00$
            & $95.00$
            & $91.75$
            & $98.00$ \\
            \midrule
            FT
            & $76.67_{\pm 7.76}$
            & $99.47_{\pm 0.52}$
            & $99.47_{\pm 0.33}$
            & $62.33_{\pm 5.79}$
            & $77.83_{\pm 3.88}$
            & $83.89_{\pm 0.41}$ \\
            GA
            & $0.00_{\pm 0.00}$
            & $16.41_{\pm 1.54}$
            & $90.64_{\pm 1.39}$
            & $0.00_{\pm 0.00}$
            & $15.17_{\pm 1.20}$
            & $84.42_{\pm 1.57}$ \\
            $\ell_1$-sparse
            & $62.00_{\pm 4.57}$
            & $98.41_{\pm 0.59}$
            & $98.81_{\pm 0.16}$
            & $59.33_{\pm 6.01}$
            & $85.17_{\pm 3.07}$
            & $89.33_{\pm 0.26}$ \\
            SCRUB
            & $1.93_{\pm 0.09}$
            & $99.98_{\pm 0.02}$
            & $99.66_{\pm 0.40}$
            & $3.00_{\pm 1.41}$
            & ${89.58_{\pm 0.84}}$
            & $94.69_{\pm 0.11}$ \\
            \midrule
            Our method
            & $0.00_{\pm 0.00}$
            & $98.50_{\pm 0.11}$
            & $98.87_{\pm 0.16}$
            & $0.67_{\pm 0.47}$
            & $89.50_{\pm 1.24}$
            & $93.22_{\pm 0.36}$ \\
            \bottomrule
        \end{tabular}
    \end{table*}

\textcolor{black}{\paragraph{Robustness to imperfect adjacency.}
To assess how sensitive our method is to imperfectly specified adjacent retain sets, we conduct a robustness study on the CIFAR-100 superclass unlearning task. Starting from the clean partition of the retain set into adjacent and remote subsets, we consider two noisy variants: (i) \emph{Case 1}, where $20\%$ of random samples from the remote retain set are mis-identified as adjacent; and (ii) \emph{Case 2}, where $20\%$ of random samples from the true adjacent retain set are mis-identified as remote. 
Table~\ref{tab:adj_noise} reports the resulting accuracies.
Our method remains robust under these perturbations: the forget accuracy stays at $0\%$ on the training data and below $6\%$ on the test data, while the changes in adjacent and remote retain accuracies are modest.
This indicates that our method does not require perfectly identified adjacency to be effective and can tolerate a reasonable amount of noise in the partition.}

\begin{table}[!htbp]
    \color{black}{
    \caption{\textcolor{black}{Robustness of our method to noisy adjacency on CIFAR-100 subclass unlearning.}}
    \label{tab:adj_noise}
    \centering
    {\resizebox{0.95\linewidth}{!}{%
    \begin{tabular}{l|ccc|ccc}
        \toprule
        & \multicolumn{3}{c}{\textbf{Training accuracy}} & \multicolumn{3}{c}{\textbf{Test accuracy}} \\
        \textbf{Setting}
        & \textbf{$\mathcal{D}_f$}
        & \textbf{$\mathcal{D}_r^{\text{adj}}$}
        & \textbf{$\mathcal{D}_r^{\text{rem}}$}
        & \textbf{$\mathcal{D}_f$}
        & \textbf{$\mathcal{D}_r^{\text{adj}}$}
        & \textbf{$\mathcal{D}_r^{\text{rem}}$} \\
        \midrule
        Clean adjacency
        & $0.00_{\pm 0.00}$
        & $98.17_{\pm 0.31}$
        & $98.44_{\pm 0.05}$
        & $2.33_{\pm 0.47}$
        & $78.17_{\pm 0.31}$
        & $81.10_{\pm 0.18}$ \\
        + $20\%$ non-adj $\rightarrow$ adj (Case 1)
        & $0.00_{\pm 0.00}$
        & $98.81_{\pm 0.20}$
        & $98.40_{\pm 0.20}$
        & $5.33_{\pm 0.58}$
        & $81.37_{\pm 0.33}$
        & $78.92_{\pm 0.95}$ \\
        + $20\%$ adj $\rightarrow$ non-adj (Case 2)
        & $0.00_{\pm 0.00}$
        & $93.93_{\pm 0.67}$
        & $95.75_{\pm 0.31}$
        & $5.00_{\pm 1.00}$
        & $77.32_{\pm 0.41}$
        & $80.08_{\pm 1.18}$ \\
        \bottomrule
    \end{tabular}%
    }}
    }
\end{table}
\textcolor{blue}{
\paragraph{kNN-Based Identification of Adjacent Retain Set}
\label{appendix:knn_adjacent}}

{\color{black}{We also studies an alternative way to construct the adjacent retain set based on feature-space similarity, instead of task-defined superclasses on CIFAR-100.}}

{\color{black}We extract output features from the pretrained ResNet-18, compute the $k$ nearest neighbors ($k=20$) of each forget sample among all retained samples, and assign every retained sample an adjacency score equal to the number of times it appears in these $k$NN lists. The top 10\% of retained samples by this score are treated as the \emph{$k$NN adjacent retain set}, and the remaining retained samples form the \emph{$k$NN remote retain set}. Our two-stage unlearning algorithm is then applied using this automatically constructed partition.

The results in Table~\ref{tab:knn_adjacent} show that the method continues to achieve strong forgetting while maintaining high accuracy on both adjacent and remote retain subsets. Overall performance is comparable to the setting where adjacency is defined by the superclass structure.}

\begin{table}[!htbp]
\centering
\caption{\textcolor{black}{Comparison of our method under task-defined adjacency vs.\ $k$NN-identified adjacency on CIFAR-100.}}
\label{tab:knn_adjacent}
{\resizebox{0.95\linewidth}{!}{\color{black}{
\begin{tabular}{l|ccc|ccc}
\toprule
Method & Train $D_f$ & Train $D_r^{\text{adj}}$ & Train $D_r^{\text{rem}}$
       & Test $D_f$ & Test $D_r^{\text{adj}}$ & Test $D_r^{\text{rem}}$ \\
\midrule
Ours (task-defined)
& $0.00_{\pm 0.00}$ & $98.17_{\pm 0.31}$ & $98.44_{\pm 0.05}$
& $2.33_{\pm 0.47}$ & $78.17_{\pm 0.31}$ & $81.10_{\pm 0.18}$ \\
Ours ($k$NN-identified)
& $3.00_{\pm 0.75}$ & $99.31_{\pm 0.31}$ & $99.85_{\pm 0.07}$
& $6.00_{\pm 0.82}$ & $77.67_{\pm 0.31}$ & $83.13_{\pm 0.23}$ \\
\bottomrule
\end{tabular}}
}}
\end{table}

\vspace{1em}
\color{black}{
\paragraph{Comparison with Retraining}
\label{appendix:retrain}

For completeness, we also compare our method with full retraining on CIFAR-100, TinyImageNet, and ToxiGen, under the same forget/retain splits as used in the main experiments. In all cases, the retrained model is obtained by training from scratch on the retained data only.

Table~\ref{tab:retrain_compare} summarizes the results. While retraining generally maintains high accuracy on the retain sets, it does not always achieve strong erasure on the forget set in our setting: the forget-set accuracy often remains relatively high. In contrast, our method consistently yields substantially lower forget accuracy while preserving competitive performance on both adjacent and remote retain subsets.}

\begin{table}[!htbp]
\centering
\caption{\textcolor{black}{Retraining vs.\ our method on CIFAR-100, TinyImageNet, and ToxiGen.}}
\label{tab:retrain_compare}
\color{black}{
\resizebox{\linewidth}{!}{
\begin{tabular}{l|l|ccc|ccc}
\toprule
Dataset & Method
& Train $D_f$ & Train $D_r^{\text{adj}}$ & Train $D_r^{\text{rem}}$
& Test $D_f$ & Test $D_r^{\text{adj}}$ & Test $D_r^{\text{rem}}$ \\
\midrule
CIFAR-100 & Retrain
& $38.40_{\pm 3.80}$ & $99.98_{\pm 0.02}$ & $99.99_{\pm 0.00}$
& $37.00_{\pm 5.10}$ & $83.92_{\pm 4.20}$ & $83.37_{\pm 0.31}$ \\
CIFAR-100 & Ours
& $0.00_{\pm 0.00}$ & $98.17_{\pm 0.31}$ & $98.44_{\pm 0.05}$
& $2.33_{\pm 0.47}$ & $78.17_{\pm 0.31}$ & $81.10_{\pm 0.18}$ \\
\midrule
TinyImageNet & Retrain
& $62.82_{\pm 5.59}$ & $99.46_{\pm 0.63}$ & $97.99_{\pm 2.01}$
& $64.45_{\pm 4.79}$ & $95.71_{\pm 0.23}$ & $90.58_{\pm 0.12}$ \\
TinyImageNet & Ours
& $0.00_{\pm 0.00}$ & $98.95_{\pm 0.08}$ & $98.49_{\pm 0.09}$
& $3.11_{\pm 0.31}$ & $91.27_{\pm 0.78}$ & $88.88_{\pm 0.54}$ \\
\midrule
ToxiGen & Retrain
& $8.58_{\pm 0.66}$ & $93.71_{\pm 2.08}$ & $91.50_{\pm 1.44}$
& $12.13_{\pm 4.34}$ & $91.74_{\pm 2.45}$ & $89.35_{\pm 2.73}$ \\
ToxiGen & Ours
& $11.95_{\pm 0.02}$ & $88.88_{\pm 0.01}$ & $92.73_{\pm 0.01}$
& $14.29_{\pm 0.06}$ & $85.86_{\pm 0.00}$ & $85.23_{\pm 0.01}$ \\
\bottomrule
\end{tabular}
}}
\end{table}

\section{Use of LLM}
In preparing this manuscript, we employed a large language model (LLM) solely to assist with refining and polishing the text. The LLM was used to improve clarity, coherence, and readability, as well as to ensure consistent terminology throughout the paper. Importantly, all technical content, experimental design, and results were independently developed and verified by the authors; the LLM did not contribute to any scientific or methodological aspects of the work.

\end{document}